\documentclass{article}
\usepackage{fullpage}

\usepackage[title]{appendix}

\usepackage{algpseudocode}

\usepackage{diagbox}
\usepackage[utf8]{inputenc} 
\usepackage[T1]{fontenc}    
\usepackage{url}            
\usepackage{multirow}
\usepackage{booktabs}       
\usepackage{amsfonts,amssymb,amsthm}       
\usepackage{nicefrac}       
\usepackage{microtype}      
\usepackage{hhline}
\usepackage{graphicx}
\usepackage{multirow}
\usepackage{tabularray}
\usepackage[authoryear]{natbib}

\usepackage{indentfirst}
\usepackage{mathtools}
\usepackage{amsmath}
\usepackage{amssymb}
\usepackage{mathtools}
\usepackage{appendix}
\usepackage{subfigure} 

\usepackage[utf8]{inputenc} 
\usepackage[T1]{fontenc}    
\usepackage{hyperref}       
\usepackage{url}            
\usepackage{booktabs}       
\usepackage{amsfonts}       
\usepackage{nicefrac}       
\usepackage{microtype}      
\usepackage{xcolor}         

\usepackage{afterpage}

\usepackage{amsfonts,enumitem,amsmath,amsthm}       
\usepackage{nicefrac}       
\usepackage{microtype}      
\usepackage{xcolor,appendix}         
\usepackage{float}          
\usepackage{graphicx}
\usepackage{subcaption}
\usepackage{amssymb} 
\usepackage[linesnumbered,ruled,vlined]{algorithm2e}

\newtheorem{proposition}{Proposition}

\usepackage{algpseudocode}

\newtheorem{theorem}{Theorem}
\newtheorem{remark}{Remark}
\newtheorem{lemma}{Lemma}

\theoremstyle{plain}
\theoremstyle{definition}

\usepackage{amsmath,amsfonts,bm}

\def\eqref#1{equation~\ref{#1}}

\def\1{\bm{1}}
\def\0{\bm{0}}

\DeclareMathAlphabet{\mathsfit}{\encodingdefault}{\sfdefault}{m}{sl}
\SetMathAlphabet{\mathsfit}{bold}{\encodingdefault}{\sfdefault}{bx}{n}

\usepackage{inputenc}
\usepackage{fontenc}
\usepackage{indentfirst}
\usepackage{cite}
\usepackage{hyperref}
\usepackage{amssymb}
\usepackage{float}
\usepackage{amsmath}
\usepackage{graphicx}
\usepackage{amsthm}

\title{\textbf{Exposing Diversity Bias in Deep Generative Models: \\ Statistical Origins and Correction of Diversity Error}}

\date{}

\renewcommand{\cite}{\citep}

\author{
Farzan Farnia\thanks{Department of Computer Science and Engineering, 
The Chinese University of Hong Kong. 
\texttt{\{farnia, mjalali24, aospanov9\}@cse.cuhk.edu.hk}}
\thanks{Authors are listed in alphabetical order.}
\and
Mohammad Jalali\footnotemark[1]
\and
Azim Ospanov\footnotemark[1]
}

\begin{document}
\maketitle

\begin{abstract}
    Deep generative models have achieved great success in producing high-quality samples, making them a central tool across machine learning applications. Beyond sample quality, an important yet less systematically studied question is whether trained generative models faithfully capture the diversity of the underlying data distribution. In this work, we address this question by directly comparing the diversity of samples generated by state-of-the-art models with that of test samples drawn from the target data distribution, using recently proposed reference-free entropy-based diversity scores, Vendi and RKE. Across multiple benchmark datasets, we find that test data consistently attains substantially higher Vendi and RKE diversity scores than the generated samples, suggesting a systematic downward diversity bias in modern generative models. To understand the origin of this bias, we analyze the finite-sample behavior of entropy-based diversity scores and show that their expected values increase with sample size, implying that diversity estimated from finite training sets could inherently underestimate the diversity of the true distribution. As a result, optimizing the generators to minimize divergence to empirical data distributions would induce a loss of diversity. Finally, we discuss potential diversity-aware regularization and guidance strategies based on Vendi and RKE as principled directions for mitigating this bias, and provide empirical evidence suggesting their potential to improve the results.
\end{abstract}

\section{Introduction}
Deep generative models have achieved substantial progress in synthesizing complex, high-dimensional data across vision, language, and audio modalities. Early generative modeling approaches including variational autoencoders (VAEs)~\citep{kingma2013vae} and generative adversarial networks (GANs)~\citep{goodfellow2014gan} established scalable paradigms for likelihood-based and divergence minimization-based data generation, while recently diffusion and score-based models~\citep{song2019generative,ho2020ddpm,song2021sde} further improved training stability by formulating generation as iterative denoising. From a statistical perspective, these models aim to learn a neural net-based sample generator that produces fresh draws approximating the data-generating distribution, without explicitly restricting the distribution to a fixed parametric family as in classical statistics.

Evaluating how well a generative model approximates the target distribution is typically addressed by standard quantitative metrics in the literature. In image generation, the Fréchet Distance (FD)~\citep{heusel2017fid,kynkaanniemi_role_2022,stein_exposing_2023} compares mean and covariance statistics with a reference data distribution, which occurs after embedding samples using pretrained feature extractors such as Inception~\citep{szegedy2016inception}, CLIP~\citep{radford_learning_2021}, and DINOv2~\citep{oquab_dinov2_2023}. The Kernel Inception Distance (KID)~\citep{binkowski2018kid} relies instead on the maximum mean discrepancy (MMD)~\citep{gretton2012mmd} measured in the same embedding spaces. These metrics are widely used for benchmarking and model comparison, but they are not designed to reveal systematic effects that arise from finite-sample training of the generative models or to isolate structural biases in learned generative distributions.

\begin{figure}[t]
    \centering
    \includegraphics[width=0.99\linewidth]{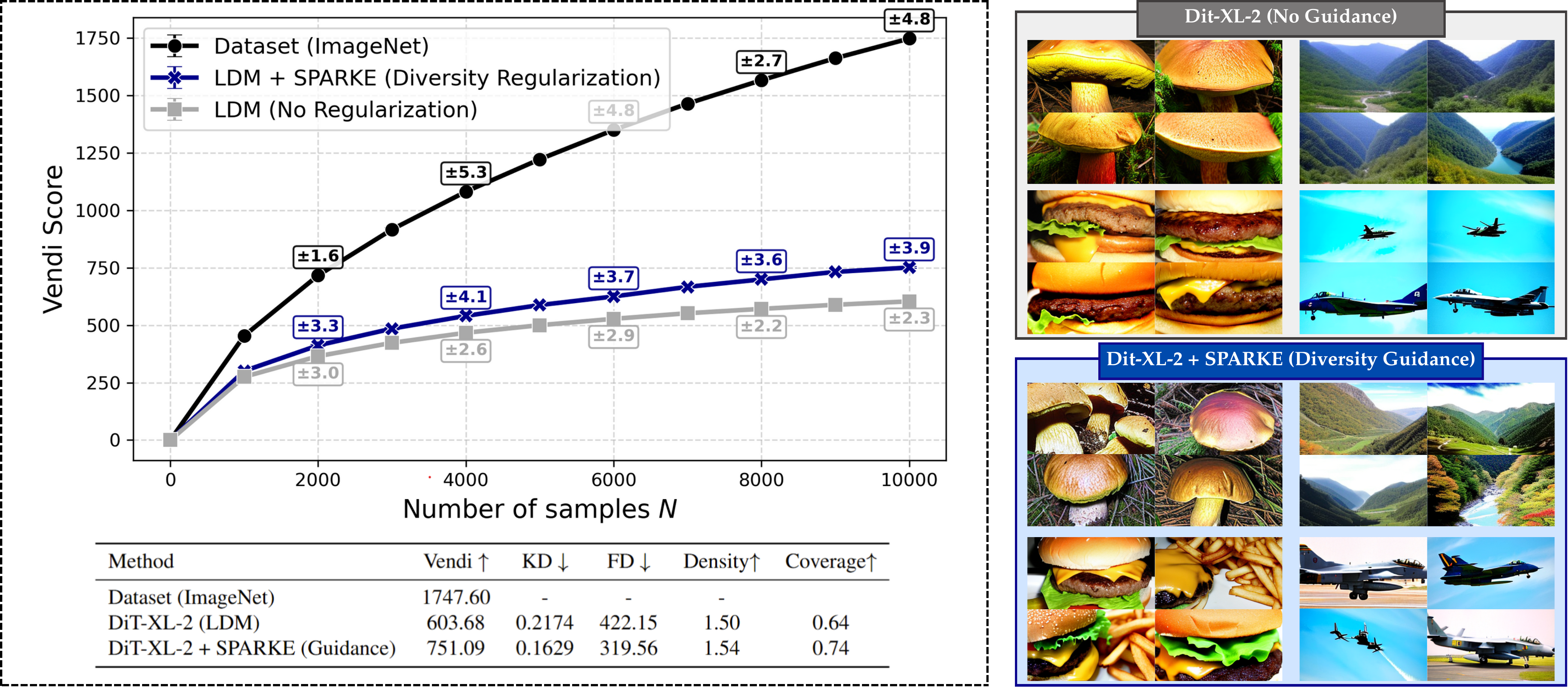}
    \caption{Numerical evaluation of Vendi scores at different sample sizes (averaged over 100 independent random trials with confidence interval) for the validation set of ImageNet, the pre-trained latent diffusion model (LDM) of DiT-XL-2, and the SPRAKE-guided LDM with entropy diversity regularization proposed in \citep{jalali2025sparke}. The entropy regularization in the SPARKE Guidance not only improves the Vendi and Coverage diversity scores, but further improves the FD and KD overall evaluation scores.}\vspace{-2mm}
    \label{fig:intro_fig}
\end{figure}

In this work, we focus on  \textit{diversity} as one such structural aspect that could be systematically biased in the trained models. Specifically, we study whether standard deep generative models can capture the diversity of the underlying data distribution, or whether they exhibit a systematic shortfall. To do so, we adopt reference-free entropy-based diversity measures, including the Vendi score~\citep{friedman2023vendi} and the Rényi Kernel Entropy (RKE) score~\citep{jalali2023rke}, which quantify diversity through spectral entropy of kernel similarity matrices computed from samples. Using these measures as well as the reference-based diversity scores Recall \citep{sajjadi_assessing_2018,kynkaanniemi_improved_2019} and Coverage \citep{naeem2020reliable}, we evaluate state-of-the-art generative models on standard benchmarks such as ImageNet~\citep{deng2009imagenet}, MS-COCO~\citep{lin2014coco}, and FFHQ~\citep{karras2019style}. Across all settings, we observe that generated samples consistently exhibit lower entropy-based diversity than the real datasets, suggesting a universal downward diversity bias.

To understand potential reasons for this bias, we specifically examine the finite size of the training dataset for generative models and whether their empirical distributions capture the full diversity of the underlying population distribution. For basic discrete random variables, the underestimation of entropy from finite empirical samples is well-known in classical statistics~\citep{miller1955bias,basharin1959entropy}. It is plausible that similar effects occur for kernel-entropy Vendi diversity measures of finite training datasets. To investigate this, we measure the Vendi score on progressively larger subsamples of real datasets. Within the range where exact spectral computation of Vendi score is feasible, we observe a pronounced increase in measured diversity with sample size, suggesting that common dataset sizes may still lie in a regime where diversity estimates have not stabilized.

This empirical behavior connects to the classical theory of entropy estimation under finite sampling. Bias analyses for plug-in entropy estimators~\citep{miller1955bias} and general nonparametric treatments of entropy estimation~\citep{paninski2003entropy} show that empirical procedures can systematically underestimate population-level entropy. While the Vendi score in \citep{friedman2023vendi} is the matrix-based spectral entropy rather than standard Shannon entropy, it still depends on a finite-sample empirical kernel operator. Building on this perspective, we prove that the expected logarithm of the Vendi score of an i.i.d. sample set is monotonically increasing with the sample size. This formalizes the intuition that finite datasets could under-represent population diversity, and it suggests one possible route by which a diversity shortfall can arise before a generative model is trained.

Beyond diagnosing the bias, understanding downward diversity bias has practical implications for model training and inference. Our theoretical results suggest that projecting a learned distribution onto suitable Vendi or RKE-based super-level sets can reduce kernel discrepancies measured by MMD when the learned model exhibits a diversity shortfall relative to the population distribution. This motivates entropy-based diversity regularization during training or fine-tuning of the generative models. This also provides a theoretical justification for diversity-aware guidance mechanisms, including recently proposed Vendi~Guidance~\citep{hemmat2024cvsg} and RKE~Guidance~\citep{jalali2025sparke}, at inference time for diffusion models.

\vspace{-3mm}
\section{Related Works}
\paragraph{Diversity and evaluation of generative models.}
Evaluation of deep generative models has been extensively studied through sample-based metrics, including the Inception Score \citep{salimans2016improved} and embedding-based distances such as FID \citep{heusel2017fid} and KID \citep{binkowski2018kid}. To separate fidelity from coverage, precision--recall-style metrics were proposed by \citet{sajjadi_assessing_2018} and refined by improved precision/recall estimators based on manifold estimation \citep{kynkaanniemi_improved_2019}, with related density and coverage scores by \citet{naeem2020reliable}. The metrics have been further updated by employing stronger backbone embeddings, including CLIP \citep{kynkaanniemi_role_2022} and DINOv2 \citep{stein_exposing_2023}. A complementary perspective on diversity limitations is provided by the birthday-paradox-based test of \citet{arora2018gans}, which estimates the effective support size of a model’s distribution, suggesting that (at the time) GAN samples could output relatively small support sets.  

\paragraph{Reference-free diversity metrics and kernel-entropy scores.}
Reference-free diversity measures quantify variability within a set of generated samples without requiring an explicit reference dataset. The Vendi score \citep{friedman2023vendi} defines diversity as the exponential of the von Neumann entropy of a kernel similarity matrix, and Rényi Kernel Entropy (RKE) \citep{jalali2023rke} provides an order-2 Rényi analogue with tractable structure and mode-count interpretations in certain settings. These kernel-entropy ideas have also been used to steer generation: contextualized Vendi score guidance (VSG) \citep{hemmat2024cvsg} injects Vendi-based gradients into latent diffusion sampling to encourage diversity (under the chosen kernel/representation) while controlling drift. Related extensions incorporate quality into diversity scoring, e.g., quality-weighted Vendi \citep{nguyen2024qwvendi}, and recent work proposes gradient-space entropy measures such as G-Vendi \citep{jung2025prismatic}, linking diversity of data to generalization in downstream reasoning tasks.

\paragraph{Novelty, memorization, rarity, and fine-grained diagnostics.}
Beyond diversity, several works target novelty and memorization behavior of generative models. \citet{alaa2022faithful} propose sample-level auditing via the Authenticity score.  \citet{jiralerspong2023fld} propose Feature Likelihood Divergence, a parametric feature-space likelihood approach to jointly evaluate novelty, fidelity, and diversity from samples. \citet{han2022rarity} propose the Rarity score to quantify per-sample uncommonness. Also, \citet{zhang2024ken} propose the Kernel-based Entropic Novelty score for identifying sample types generated more frequently.

\paragraph{Diversity in diffusion models.} Recent work has examined inference-time mechanisms for improving diversity in diffusion models. \citet{sadat2024cads} propose CADS which injects scheduled noise into the conditioning signal during inference to trade off conditioning strength and diversity. \citet{corso2024particle} develop Particle Guidance for non-i.i.d. diversity-aware sampling, guiding a set of simultaneously generated samples to increase diversity via particle-style interactions. Also, Vendi and RKE-based guidance by \citet{hemmat2024cvsg,jalali2025sparke} use matrix-based entropy as an explicit diversity objective. Our work examines the role of entropy-guided methods in addressing diversity biases in generative models.

\paragraph{Entropy, kernel operators, and bias in entropy estimation.}
Kernel/operator entropy functionals connect generative-model diversity evaluation to information theory and quantum information, where von Neumann entropy is defined on positive semidefinite operators. In machine learning, matrix-based and kernel-based entropy estimators were developed to avoid explicit density estimation, including kernel constructions motivated by Rényi-style axioms \citep{sanchezgiraldo2012entropy} and RKHS information-theoretic frameworks based on covariance operators \citep{bach2022itkernel}. Separately, classical statistics shows that plug-in entropy estimation from finite samples can exhibit systematic downward bias, starting from early bias analyses \citep{miller1955bias} and nonparametric treatments \citep{paninski2003entropy}, with modern minimax-optimal estimation of distribution functionals including entropy \citep{jiao2015minimax} and improved estimators for the unseen regime \citep{valiant2013unseen}. These results inspire the study of finite-sample effects when diversity is quantified through entropy-like functionals.

\section{Preliminaries}
Consider a data random variable $X\in \mathcal{X}$ distributed according to $P_X$ over the sample space $\mathcal{X}$.
Then, a trained generative model $\mathcal{G}$ induces a probability distribution $P_\mathcal{G}$ on variable $X$. In a reference-free evaluation process, the evaluator has only access to $n$ independently generated samples from $P_\mathcal{G}$, denoted by $x_1,\ldots,x_n \in \mathcal{X}$. 
In our analysis, we utilize entropy-based metrics that quantify the diversity of the drawn sample set $x_1,\ldots,x_n$ in a reference-free manner without requiring access to a reference distribution. In our evaluation of image generative models, we use DINOv2-embedded data following the reference \citep{stein_exposing_2023}.

\subsection{Kernel Functions and Kernel Matrices}
Following the standard definition, a function $k:\mathcal{X}\times\mathcal{X}\rightarrow\mathbb{R}$ is called a kernel if for every $n\in\mathbb{N}$ and inputs $x_1,\ldots,x_n\in\mathcal{X}$, the corresponding Gram matrix $K\in\mathbb{R}^{n\times n}$ with entries $K_{ij}=k(x_i,x_j)$ is positive semidefinite (PSD).
By the classical characterization due to \citet{aronszajn1950reproducing}, there exists a (possibly infinite-dimensional) reproducing kernel Hilbert space $\mathcal{H}$ and a feature map $\phi:\mathcal{X}\rightarrow\mathcal{H}$ such that
\begin{equation}\label{eq:kernel_feature_map}
    k(x,x') = \bigl\langle \phi(x), \phi(x') \bigr\rangle_{\mathcal{H}}.
\end{equation}

In this work, we focus on \textit{normalized kernels} satisfying $k(x,x)=1$ for all $x\in\mathcal{X}$.
Common examples include cosine similarity applied to unit-normalized embeddings and the Gaussian (RBF) kernel.
For a normalized kernel, the Gram matrix satisfies $\mathrm{Tr}(K)=\sum_{i=1}^n k(x_i,x_i)=n$, and therefore the normalized kernel matrix $\frac{1}{n}K$ has unit trace.
Since $\frac{1}{n}K$ is PSD, its eigenvalues are nonnegative and sum to one, i.e. they form a probability model.

\subsection{Kernel Covariance and Population Operator}
Let $\Phi\in\mathbb{R}^{n\times d}$ denote the feature matrix whose $i$-th row is $\phi(x_i)^\top$ in the finite-dimensional case (and the same identities hold at the operator level in general).
Then the normalized kernel matrix can be written as
\begin{equation}
    \frac{1}{n}K = \frac{1}{n}\Phi\Phi^\top.
\end{equation}
Define the empirical kernel covariance matrix $\widehat{C}_X\in\mathbb{R}^{d\times d}$:
\begin{equation}
    \widehat{C}_X := \frac{1}{n}\sum_{i=1}^n \phi(x_i)\phi(x_i)^\top = \frac{1}{n}\Phi^\top\Phi.
\end{equation}
The matrices $\frac{1}{n}\Phi\Phi^\top$ and $\widehat{C}_X$ share the same non-zero eigenvalues (singular values of $\frac{1}{\sqrt{n}}\Phi$), and thus any spectral functional applied to $\frac{1}{n}K$ can equivalently be applied to $\widehat{C}_X$.

At the population level, we define the kernel covariance operator
\begin{equation}
    \widetilde{C}_X := \mathbb{E}_{X\sim P_X}\Bigl[\phi(X)\phi(X)^\top\Bigr].
\end{equation}
Under the normalized kernel assumption, $\widetilde{C}_X$ is positive semidefinite and satisfies $\mathrm{Tr}(\widetilde{C}_X)=\mathbb{E}[k(X,X)]=1$, thus it is a density matrix whose eigenvalues form a valid probability model.

\subsection{Matrix-based Entropy: Vendi and RKE Diversity Scores}
Let $A\in\mathbb{R}^{d\times d}$ be a PSD matrix with unit trace $\mathrm{Tr}(A)=1$, and let $\{\lambda_i\}_{i=1}^d$ denote its eigenvalues.
The \textit{von Neumann entropy} of $A$ is defined as
\begin{equation}\label{eq:von_neumann_entropy}
    H(A) := \sum_{i=1}^d \lambda_i \log\frac{1}{\lambda_i}
\end{equation}

Given samples $x_1,\ldots,x_n$ and their normalized kernel matrix $\frac{1}{n}K$, \citet{friedman2023vendi} define the \textit{Vendi score} as the exponential of the von Neumann entropy,
\begin{equation}\label{eq:vendi_score}
    \mathrm{Vendi}(x_1,..,x_n)\! :=\! \exp\bigl(H\bigl(\frac{1}{n}K\bigr)\bigr) =\exp\bigl(H(\widehat{C}_X)\bigr)
\end{equation}
Following the analysis in \citep{bach2022itkernel,ospanov2025vendi}, the population Vendi score of a distribution $P_X$ is defined analogously by applying the same construction to the population kernel covariance operator $\widetilde{C}_X$,
\begin{equation}
    \mathrm{Vendi}(P_X) := \exp\bigl(H(\widetilde{C}_X)\bigr).
\end{equation}
\begin{remark}
As discussed in \citep{friedman2023vendi,ospanov_fkea_2024}, the \emph{exact} computation of the Vendi score with a Gaussian kernel (and more generally infinite-dimensional kernels) requires an eigendecomposition of the $n\times n$ kernel matrix, which becomes computationally prohibitively expensive for dataset sizes larger than $\approx 20000$. Therefore, our Vendi score evaluations remain constrained to 20000 samples,  and we report the Vendi score plots at sample sizes below this threshold. For image data, we use DINOv2 backbone embedding following \citep{stein_exposing_2023}. 

Also, as we prove in the Appendix (Proposition~\ref{prop:vendi_concentration_single_eval_app}), the Vendi score $\mathrm{Vendi}(x_1,\ldots ,x_n)$ of an instance of $n$ i.i.d. samples $x_1,\ldots ,x_n\stackrel{\text{iid}}{\sim} P_X$ concentrates around the expected value $\mathbb{E}[\mathrm{Vendi}(x_1,\ldots ,x_n)]$ where the expectation is over the randomness of $n$ i.i.d. samples from $P_X$ (this is different from the Vendi score of the population distribution $\mathrm{Vendi}(P_X)$). To further ensure our reported scores are properties of the underlying distribution, we perform the Vendi score evaluation for $M=10$ independent trials and report the averaged score and its confidence interval.  
\end{remark}
Also, note that the \textit{RKE score} \citep{jalali2023rke} corresponds to order-$2$ Rényi entropy of the kernel matrix, which for  samples $x_1,\ldots,x_n$ and their kernel matrix $K$  will be ($\Vert\cdot \Vert_F$ denotes the Frobenius norm):
\begin{equation}\label{eq:rke_score}
    \mathrm{RKE}(x_1,\ldots,x_n) := \Bigl\Vert\frac{1}{n}K\Bigr\Vert_F^{-2}.
\end{equation}
As discussed in \citep{jalali2023rke}, the population RKE admits a simple characterization: if $X,X'$ are i.i.d.\ samples from $P_X$, then
$\mathrm{RKE}(P_X) = {1}/{\mathbb{E}\bigl[k^2(X,X')\bigr]}$. This identity follows from interpreting the order-$2$ kernel entropy through the Hilbert--Schmidt norm of the population kernel covariance operator.

\section{Finite-Sample Entropy Bias and its Effects on Diversity in Generative Models}
\label{sec:entropy_bias_to_vendi}

\begin{figure}
    \centering
    \includegraphics[width=\linewidth]{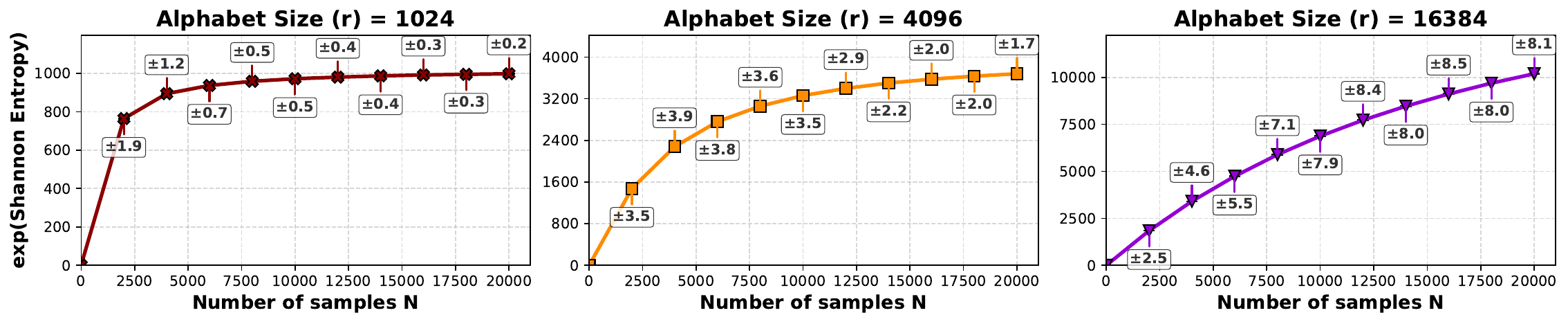}
    \caption{Exponential of Shannon entropy vs. sample size $N$ for varying alphabet sizes $R$. Error bars represent 95\% confidence intervals calculated over 10 independent trials.}
    \label{fig:entropy bias}
\end{figure}

\begin{figure}[h]
    \centering
    \includegraphics[width=\textwidth]{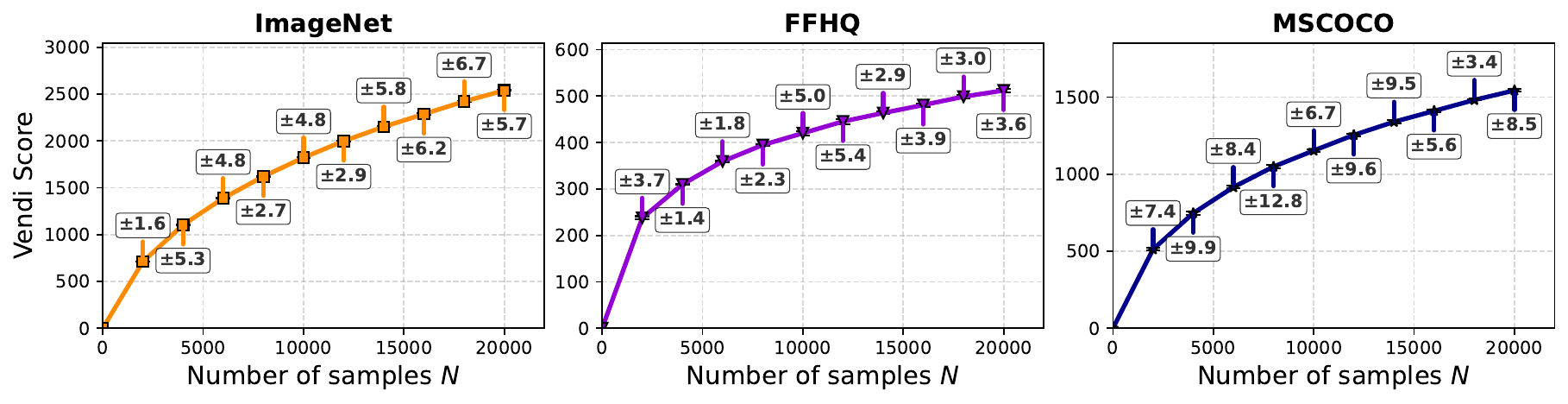}
    \caption{Vendi score curves (mean and confidence intervals over 10 independent sample sets) for ImageNet, FFHQ, and MSCOCO (sample size values of $n\le 20\text{K}$), computed using DINOv2 embeddings and a Gaussian (RBF) kernel with bandwidth $\sigma=35$. The Vendi scores continue to increase at a significant rate across all sample sizes below the 20000-computational-feasible size limit for exact Vendi computation.}
    \label{fig:ci-dinov2-sigma35}
\end{figure}

A central question in this work is how reliably entropy-based diversity scores of finite samples $x_1,\ldots, x_n$
reflect the true entropy score of the underlying data distribution $P_X$.
This question is particularly relevant for generative modeling, where the generative model is trained on finite datasets.
If the expected entropy score of the empirical distribution $\widehat{P}_n$ of $n$ i.i.d.\ samples from $P$ exhibits a noticeable downward bias relative to the entropy of population distribution $P$ (entropy of infinite samples from $P$), then the finite training set would expose the model to less diversity than the true diversity of $P$.  
Therefore, such an entropy-based diversity gap between the empirical distribution of training data $\widehat{P}_n$ and the underlying $P$ can plausibly propagate to trained generative models on the empirical samples.

In this section, we aim to analyze this phenomenon from a statistical perspective. We begin by reviewing the classical downward bias of the Shannon entropy of finite sample sets for basic discrete random variables.
We then extend the same line of reasoning to the von Neumann entropy of the empirical kernel covariance operator, which underlies the Vendi score.
Two controlled experiments, one discrete and one continuous, are used to illustrate how entropy-based scores increase with the evaluation sample size $n$ and exhibit systematic downward bias at moderate $n$.
These observations provide the conceptual foundation for the next section, where we empirically study diversity bias in pretrained generative models.

\subsection{Classical background: downward bias of finite-sample Shannon entropy}
\label{subsec:plugin_shannon_bias}

Let $P$ be a discrete distribution on an alphabet of size $r$, and let $\widehat P_n$ denote the empirical distribution formed by $n$ i.i.d.\ samples from $P$.
The plug-in (maximum-likelihood) estimator of Shannon entropy is $H(\widehat P_n)$, i.e., the Shannon entropy of the $n$ empirical samples from $P$.
It is well-known in statistics and information theory that this estimator is \textit{downward biased}, as
\begin{equation}\label{eq:plugin_entropy_bias}
    \mathbb{E}\bigl[H(\widehat P_n)\bigr] \:\le\: H(P).
\end{equation}
Note that the above inequality follows directly from the well-known concavity of Shannon entropy and Jensen's inequality as
$\mathbb{E}[\widehat P_n]=P$.
Beyond the negative sign of this bias, sharp asymptotic expansions are already known.
In particular, for fixed $r$ and large $n$, the leading bias term scales as $\frac{r-1}{2n}$ (in nats), i.e.
\begin{equation}
    H(P) - H(\widehat P_n)\; \gtrsim\; \frac{r-1}{2n},
\end{equation}
as shown by \citet{miller1955bias} and \citet{basharin1959entropy} and refined in
\citep{paninski2003entropy,schurmann2004bias}.

\paragraph{Experiment 1: discrete entropy and alphabet-size effects.}
To visualize this phenomenon, we consider a uniform random variable on $\{1,2,\ldots,r\}$ for
$r\in\{2^{10},2^{12}, 2^{14}\}$.
Note that the entropy of the population distribution (uniform) is well-known to be $H(P)=\log r$.
For each $r$, we repeat the following experiment $M=10$ times and present the averaged result and their confidence intervals:
we sample an i.i.d.\ dataset of size $n$ until 20000,   
compute the plug-in entropy $H(\widehat P_n)$, and average the results across $M$ trials.
To facilitate comparison across different values of $r$, we exponentiate the entropies and plot
$\exp(H(\widehat P_n))$.
The resulting curves in Figure~\ref{fig:entropy bias} exhibit an increasing trend with $n$ and a downward gap at moderate sample sizes, with the gap becoming more pronounced as $r$ increases.

\subsection{From Shannon entropy to log-Vendi: a kernel-based analogue}
\label{subsec:logvendi_bias}

\begin{figure}
    \centering
    \includegraphics[width=\linewidth]{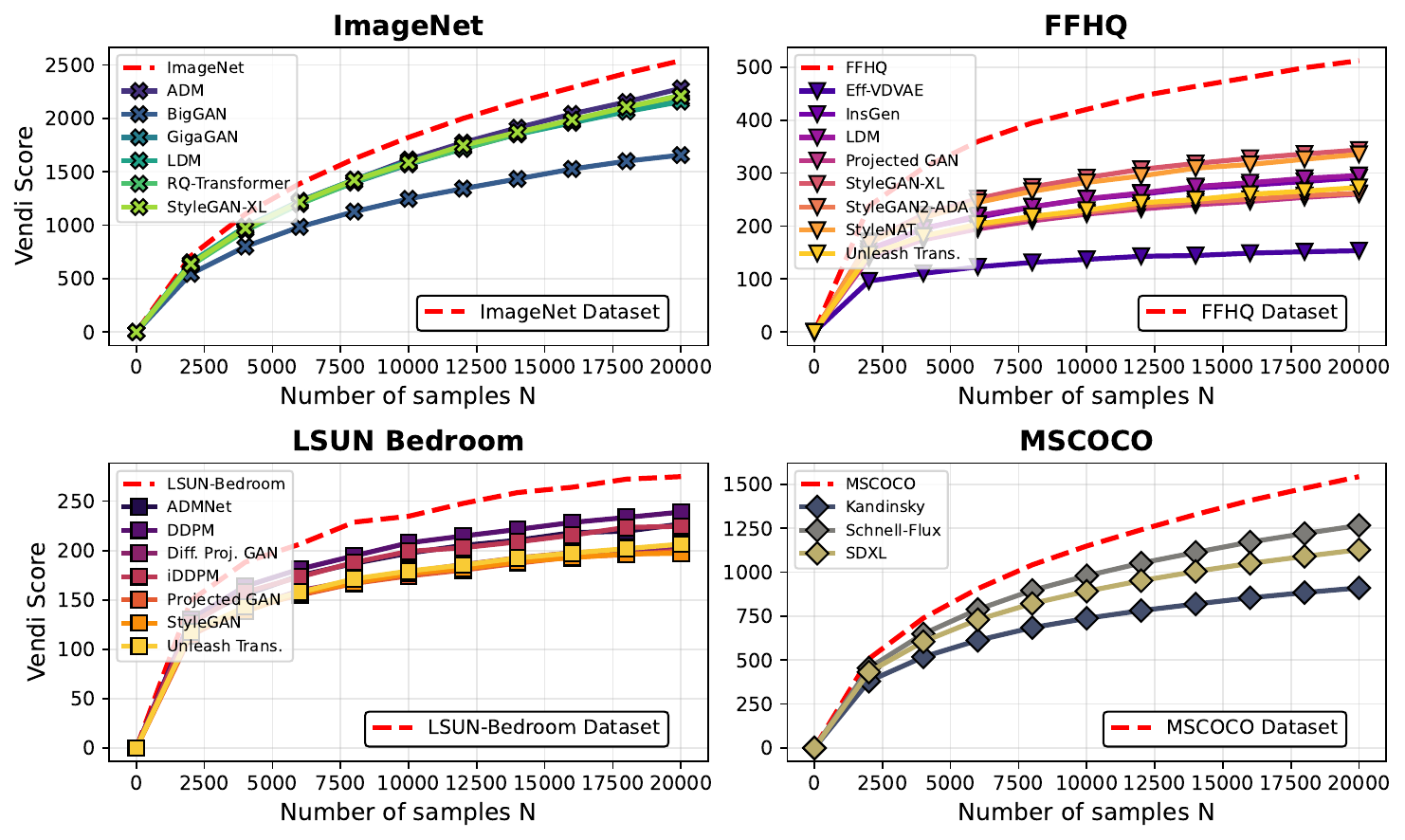}
    \caption{Comparison of Vendi scores of the test sample set (the dashed red curve) and the generated samples by pre-trained generative models across four datasets. The backbone embedding is DINOv2 embeddings using Gaussian (RBF) kernel with bandwidth $\sigma=35$.}
    \label{fig:models-comparison-2x2}\vspace{-3mm}
\end{figure}

We now extend the above intuition to the matrix-based entropy underlying the Vendi score.
Recall the normalized kernel feature map $\phi:\mathcal{X}\to\mathcal{H}$ satisfying
$k(x,x')=\langle \phi(x),\phi(x')\rangle_{\mathcal{H}}$ and $\|\phi(x)\|_{\mathcal{H}}^2=k(x,x)=1$.
Given i.i.d.\ samples $X_1,\ldots,X_n\sim P$, define the empirical kernel covariance operator
\begin{equation}\label{eq:empirical_cov_operator}
    \widehat{C}_n := \frac{1}{n}\sum_{i=1}^n \phi(X_i)\phi(X_i)^\top,
\end{equation}
which is positive semidefinite and has unit trace.
The von Neumann entropy $H(\widehat{C}_n) = -\mathrm{Tr}\big(\widehat{C}_n \log \widehat{C}_n\big)$ 
equals the logarithm of the Vendi score,
$\mathrm{Vendi}(X_{1:n})=\exp(H(\widehat{C}_n))$.
We therefore refer to $H(\widehat{C}_n)$ as \textit{log-Vendi}.

\begin{proposition}[Monotone increase of expected log-Vendi]
\label{prop:monotone_logvendi}
For i.i.d.\ samples $X_1,\ldots,X_n\stackrel{\text{iid}}{\sim} P$, the sequence $ \mathbb{E}\bigl[H(\widehat{C}_n)\bigr] = \mathbb{E}\bigl[\log\bigl(\mathrm{ Vendi}(X_1,\ldots ,X_n)\bigr)\bigr] $
increases monotonically in sample size $n$, i.e.,
\begin{equation*}
    \forall m,n\in\mathbb{N} :\quad  \; m\le n\;\; \Longrightarrow \;\; \mathbb{E}\bigl[H(\widehat{C}_m)\bigr] \le \mathbb{E}\bigl[H(\widehat{C}_n)\bigr] 
\end{equation*}
Note that the expectation is with respect to the randomness of $n$ i.i.d. drawn samples from $P$.
\end{proposition}

\begin{proof}We defer the proof to the Appendix.\end{proof}

\paragraph{Experiment 2: a mixture model with an explicit effective support.}
To connect the discrete intuition of Experiment~1 with kernel-based log-Vendi, we consider a continuous Gaussian mixture distribution with a latent discrete structure.
We work in $\mathbb{R}^d$ for $d=10,12,14$ and define a Gaussian mixture with $r=2^d$ uniformly-weighted Gaussian components ($\frac{1}{r}$ weight for each component) centered at the $2^d$ vertices of the $ d$-dimensional cube $[-1,1]^d$.
Each component has isotropic covariance $\sigma^2 I$ with $\sigma=10^{-4}$, yielding $r$ separated clusters.
We evaluate log-Vendi using a Gaussian kernel with bandwidth $\tau=0.1$.
For each $r\in\{2^{10},2^{12},2^{14}\}$ and the same grid of $n$ values as in Experiment~1, we repeat the evaluation $M=10$ times
and plot $\exp(H(\widehat{C}_n))$ (i.e., the Vendi score) with confidence intervals.
The resulting curves in Figure~\ref{fig:gaussian mixture bias} (in the Appendix) display the same qualitative behavior observed in the discrete case:
Vendi increases with $n$ and exhibits a systematic downward gap at moderate sample sizes, with a larger gap for greater $r$.\vspace{1mm}

The results in this section suggest that the empirical distribution $\widehat{P}_n$ of finite $n$ samples could exhibit an entropy gap compared to the true distribution $P$. In the next section, we numerically analyze this entropy gap, which can cause a downward diversity bias in trained generative models over finite datasets relative to the underlying distribution of the test data.

\section{Evaluating Downward Entropy Bias in Generative Models}
\label{sec:eval_downward_bias}

We empirically evaluate whether entropy-based diversity scores reveal a systematic gap between real datasets and samples produced by standard pretrained generative models.
Following Section~\ref{sec:entropy_bias_to_vendi}, we frame diversity evaluation as a finite-sample statistical problem:
for any distribution $P$, either the empirical distribution induced by a dataset or a learned model distribution, we consider the expected Vendi score of an i.i.d.\ sample of size $n$ drawn from $P$.
Our goal is to compare these expected scores under a matched evaluation protocol.\vspace{1mm}

We focus on ImageNet~\citep{deng2009imagenet}, FFHQ~\citep{karras2019style}, and MS-COCO~\citep{lin2014coco}, and use pretrained models and generated samples from the \texttt{dgm-eval} repository~\citep{stein_exposing_2023}.
For each dataset and each model, finite-dataset-size expected Vendi scores are estimated by averaging the Vendi score over $5$ independently sampled subsets of each size $n$, and we report the mean and 95\% confidence interval.\vspace{1mm}

Figures~\ref{fig:ci-dinov2-sigma35} and \ref{fig:models-comparison-2x2} highlight two consistent empirical patterns: 1)
For each dataset, the Vendi score of the real dataset samples increases substantially with the evaluation sample size $n$ across the computationally feasible range (up to 20000). 2) Across all evaluated $n$ values, the Vendi curves of samples generated by standard pretrained models remain consistently below the corresponding curve of the dataset itself.
This separation persists across ImageNet, FFHQ, and MS-COCO.

To complement Vendi, we also report more scalable diversity metrics, including RKE (order-2 kernel entropy in the same embedding) as well as Recall and Coverage metrics implemented in \texttt{dgm-eval}.
Across the same datasets and pretrained models, these metrics similarly indicate that generated samples do not match the diversity of true dataset samples, supporting what is suggested by the Vendi curves. Detailed results and further analysis are provided in Tables~\ref{tab:imagenet_rke_vs_recall_coverage}, \ref{tab:ffhq_rke_vs_recall_coverage} and \ref{tab:lsun_rke_vs_recall_coverage} found in the Appendix.

Finally, we examine whether training-set size contributes to the observed diversity gap.
We train StyleGAN-XL~\citep{sauer2022styleganxl} on ImageNet using subsets of sizes $1{,}000{,}000$, $500{,}000$, $100{,}000$, and observe that the Vendi diversity of generated samples increases markedly with the size of the training set.
A similar trend is observed for an unconditional latent diffusion model~\citep{rombach2022ldm} trained on FFHQ subsets of sizes $70{,}000$, $35{,}000$, and $17{,}500$.
These controlled experiments support the view that limited training-set size can be a substantive contributor to the entropy-based diversity gap observed between real datasets and standard pretrained generative models.

\section{Entropy-Based Diversity Regularization for Generative Models}
\label{sec:entropy_regularization}

The preceding sections reveal a consistent entropy-based diversity gap:
under the same reference-free evaluation protocol, the Vendi/RKE scores of samples
from a learned model distribution $Q_\theta$ can fall below those of samples from the data distribution $P$.
Motivated by the monotone growth of the expected log-Vendi score with sample size,
we view this gap through a regularization lens:
if $P$ lies in a high-entropy region of the distribution space,
then encouraging $Q_\theta$ to move toward that region is a principled way
to reduce an entropy-defined notion of diversity error.

\textbf{High-entropy regions.}
Let $\mathcal{P}$ be a convex set of probability distributions on $\mathcal{X}$,
and let $H:\mathcal{P}\to\mathbb{R}$ be a concave entropy functional.
For a target level $\rho\in\mathbb{R}$, define the entropy superlevel set
\begin{equation}\label{eq:entropy_superlevel_set_main}
\mathcal{C}_\rho := \bigl\{Q\in\mathcal{P}:\: H(Q)\ge \rho\bigr\}.
\end{equation}
By concavity of $H$, $\mathcal{C}_\rho$ is a convex set, i.e. for every $Q_1,Q_2\in \mathcal{C}_\rho$ and $\alpha\in [0,1]$, then $\alpha Q_1 + (1-\alpha)Q_2\in \mathcal{C}_\rho$.
In discrete settings, $H$ corresponds to Shannon entropy.
In our kernelized setting, $H$ corresponds to the population von Neumann entropy
of the kernel covariance operator (population log-Vendi), which is concave in the distribution.

\begin{figure}[t]
    \centering
    \includegraphics[width=0.7\linewidth]{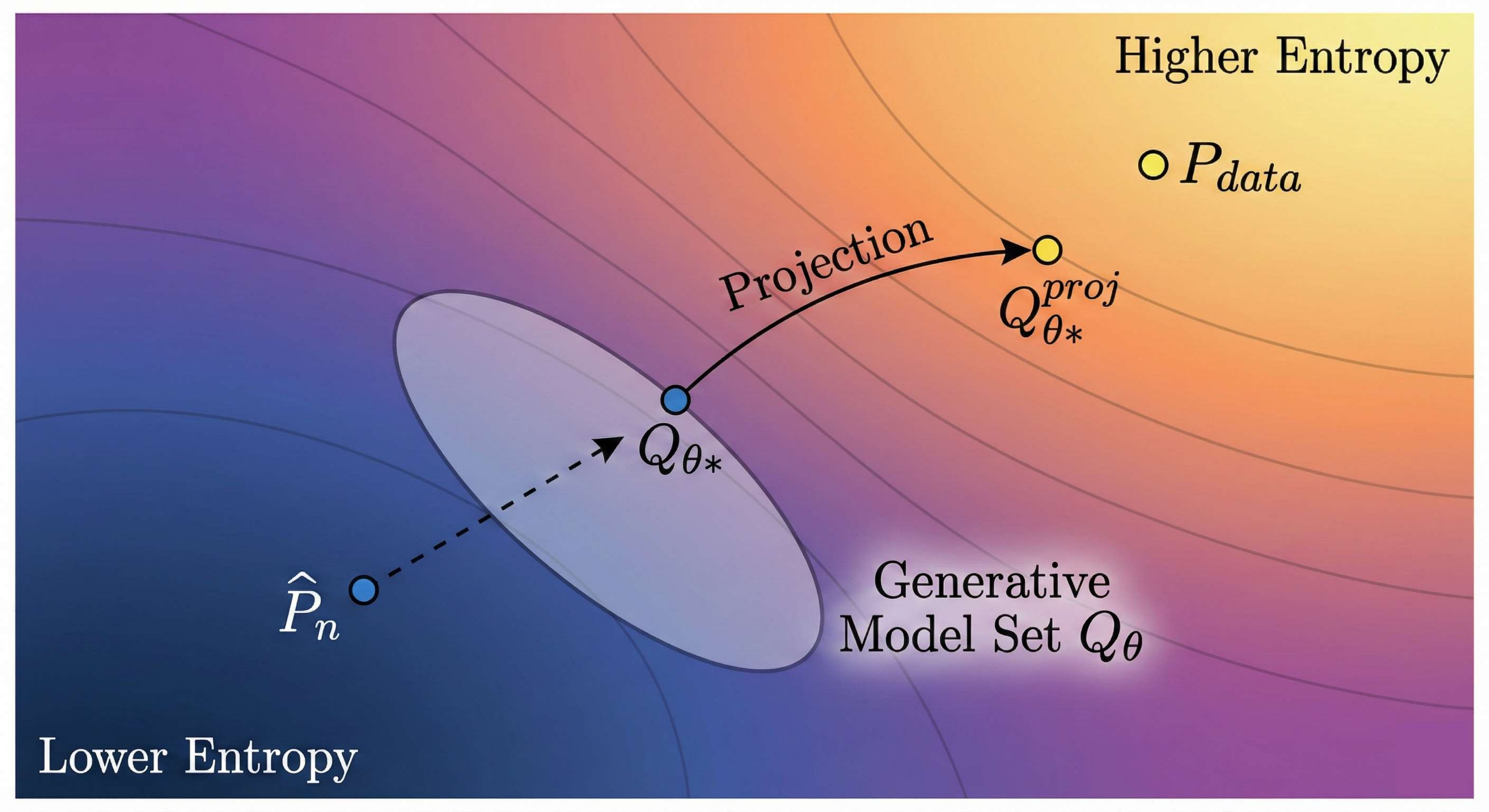}
    \caption{Entropy level sets in distribution space. $\widehat{P}_n$ and $Q_{{\theta}^*}$ lie in a lower-entropy region; $P_{\mathrm{data}}$ lies higher. Projection onto the entropy superlevel set yields $Q_{{\theta}^{\mathrm{proj}}}$, closer to $P_{\mathrm{data}}$ within the model family.
}
    \label{fig:projection-principle}
\end{figure}

\textbf{Projection principle.}
The key implication of the convexity of $\mathcal{C}_\rho$ distribution set is that projecting a learned distribution onto this high-entropy region
will improve its proximity to any target distribution already in the region,
under common discrepancy families used in generative modeling.

\begin{theorem}[Entropy projection principle]\label{thm:entropy_projection_main}
Assume $\mathcal{C}_\rho$ is a nonempty and convex set.
Let $P\in \mathcal{C}_\rho$ and let $Q_\theta\in\mathcal{P}$.

\emph{(i) Hilbertian distances.}
If $\mathrm{dist}$ is a Hilbertian distance (including MMD and KD),
and $Q^\star$ is the metric projection of $Q_\theta$ onto $\mathcal{C}_\rho$, i.e.
$Q^\star\in\arg\min_{Q\in\mathcal{C}_\rho}\mathrm{dist}(Q_\theta,Q)$,
then
\begin{equation}\label{eq:hilbert_contract_main}
\mathrm{dist}(P,Q^\star)\le \mathrm{dist}(P,Q_\theta).
\end{equation}
\emph{(ii) Bregman divergences.}
If $D_\Phi$ is a Bregman divergence (in particular, the KL divergence),
and $Q^\star$ is the corresponding I-projection
$Q^\star\in\arg\min_{Q\in\mathcal{C}_\rho} D_\Phi(Q,Q_\theta)$,
then
\begin{equation}\label{eq:bregman_contract_main}
D_\Phi(P,Q^\star)\le D_\Phi(P,Q_\theta).
\end{equation}
\end{theorem}
\begin{proof}
    We defer the proof to the Appendix.
\end{proof}

Theorem~\ref{thm:entropy_projection_main} identifies $\mathcal{C}_\rho$ as a meaningful high-entropy region:
if the target distribution lies in $\mathcal{C}_\rho$, then moving $Q_\theta$ toward $\mathcal{C}_\rho$
cannot increase its discrepancy to $P$ under standard training/evaluation distances.
Formal definitions and proofs are provided in Appendix~\ref{app:entropy_projection_proofs}. The projection process is illustrated in Figure~\ref{fig:projection-principle}.

\subsection{Practical regularization, guidance, and empirical-support entropy projection}
\label{sec:practical_projection}

This subsection translates the projection principle of Theorem~\ref{thm:entropy_projection_main} into concrete, implementable procedures. We first discuss two mechanisms that act \emph{during} model construction, that are training-time regularization for GAN-style objectives and sampling-time guidance for diffusion models. We then present a complementary mechanism that acts \emph{post} sampling: an empirical-support entropy projection that reweights a fixed set of generated samples to satisfy an entropy (Vendi/VNE) or order-2 (RKE) diversity constraint while remaining close to the baseline empirical measure under a kernelized Hilbertian geometry.

\paragraph{Training-time regularization.}
A direct way to promote entropy-based diversity is to augment a divergence-based training objective $\mathcal{L}(\theta)$ with an entropy term,
\begin{equation}\label{eq:entropy_regularized_objective_final}
\min_\theta\; \mathcal{L}(\theta) - \lambda\, H_{\mathrm{VNE}}(Q_\theta),\qquad \lambda>0.
\end{equation}
In our kernelized setting, $H_{\mathrm{VNE}}(Q_\theta)$ corresponds to the population log-Vendi, i.e., the von Neumann entropy (VNE) of a kernel covariance operator. Since VNE involves eigendecomposition, we also consider the order-2 surrogate inverse-RKE \citep{jalali2023rke}
\begin{equation}\label{eq:irke_def_final}
\mathrm{Inverse\text{-}RKE}(Q_\theta)
:= \mathbb{E}_{X,X'\sim Q_\theta}\bigl[k(X,X')^2\bigr],
\end{equation}
with $\mathrm{RKE}(Q_\theta)=1/\mathrm{Inverse\text{-}RKE}(Q_\theta)$.

\paragraph{Sampling-time guidance (diffusion/latent diffusion).}
For diffusion and latent diffusion models, entropy-aligned corrections can be applied along the reverse-time trajectory. Let $\mathcal{S}_\psi$ produce an intermediate proposal
\begin{equation}\label{eq:generic_sampler_final}
\tilde z_{t-1} = \mathcal{S}_\psi(z_t,t).
\end{equation}
Given a memory bank $\mathcal{M}=\{z^{(i)}\}_{i=1}^m$, define
\begin{equation}\label{eq:irke_latent_final}
\mathcal{J}_{\mathrm{Inverse\text{-}RKE}}(z;\mathcal{M})
:= \frac{1}{m}\sum_{i=1}^m k(z,z^{(i)})^2,
\end{equation}
leading to the guided update
\begin{equation}\label{eq:diffusion_irke_guidance_final}
z_{t-1}
= \tilde z_{t-1}
- \eta_t \nabla_z \mathcal{J}_{\mathrm{Inverse\text{-}RKE}}(\tilde z_{t-1};\mathcal{M}),
\end{equation}
with $\eta_t\ge 0$. This is the prompt-free analogue of SPARKE-guidance \citep{jalali2025sparke}; related Vendi/VNE-based guidance has also been explored in \citep{hemmat2024cvsg}.

\paragraph{Empirical-support post-hoc projection via reweighting.}
Beyond training-time regularization and sampling-time guidance, we consider a post-hoc mechanism that directly instantiates the geometric ``projection-to-high-entropy'' perspective on a fixed sample set. Given $M$ samples $\{x_i\}_{i=1}^M$ generated from a trained model $Q_\theta$, we form a reweighted empirical measure
\[
Q_q=\sum_{i=1}^M q_i\,\delta_{x_i},
\qquad q\in\Delta_M:=\Bigl\{q\ge 0,\ \sum_{i=1}^M q_i=1\Bigr\},
\qquad q_0=\frac{1}{M}\mathbf{1}.
\]
This reweighting preserves the support and only redistributes probability mass across already-generated samples.

The projection viewpoint suggests choosing $q$ to (i) remain close to the baseline empirical model $Q_{q_0}$ under a Hilbertian discrepancy (e.g., MMD/KD), while (ii) enforcing a target diversity level. Concretely, we consider the empirical-support projection program
\begin{equation}\label{eq:main_posthoc_projection}
\begin{aligned}
\min_{q\in\Delta_M}\quad & \mathrm{KD}\left(Q_q,Q_{q_0}\right) := (q-q_0)^\top K(q-q_0)\\
\text{s.t.}\quad & H_{\mathrm{VNE}}(Q_q):= H_{\mathrm{VNE}}\bigl(\mathrm{diag}(q)^{1/2} K \mathrm{diag}(q)^{1/2}\bigr) \ge\ \rho,
\end{aligned}
\end{equation}
where $\mathrm{KD}$ denotes the kernel distance induced by $k$ (equivalently, squared MMD on the fixed support), which for the kernel matrix $K$ of generated samples $x_1,\ldots , x_n$ will be $(q-q_0)^\top K(q-q_0)$. The above formulation also motivates the following Lagrangian log-Vendi $H_{\mathrm{VNE}}$ penalty variant for a regularization coefficient $\lambda > 0$:
\begin{equation}\label{eq:main_posthoc_projection_vendi_lagrangian}
\min_{q\in\Delta_M}\quad \mathrm{KD}\left(Q_q,Q_{q_0}\right)\ -\ \lambda\,H_{\mathrm{VNE}}(Q_q)
\end{equation}
As an efficient order-2 alternative aligned with the RKE score, one can also use a penalized variant based on the inverse-RKE score,
\begin{equation}\label{eq:main_posthoc_projection_rke}
\min_{q\in\Delta_M}\quad \mathrm{KD}\left(Q_q,Q_{q_0}\right)\ +\ \lambda\,\mathrm{Inverse\text{-}RKE}(Q_q)
\end{equation}
Both formulations implement the same principle: adjust the probability mass over the generated samples to reach a higher-diversity region while staying close to the original empirical model. Full technical definitions, convexity properties, and optimization algorithms are presented in Appendix~\ref{app:empirical_support_projection}.

\section{Numerical Results}

\paragraph{Model comparison.} Extending Figure~\ref{fig:ci-dinov2-sigma35}, we evaluated diverse architectures on ImageNet, FFHQ, LSUN, and MSCOCO. Figure~\ref{fig:models-comparison-2x2} confirms a systematic downward diversity bias: generative models consistently rank below real data across all sample sizes and embeddings. We provide additional results in the Appendix.

\paragraph{Training models with a restricted set of datapoints.}
To further validate our findings, we trained three versions of StyleGAN-XL using the full training set, 50\%, and 10\% of the data, respectively. As shown in Figure~\ref{fig:stylegan-xl-train}, model diversity drops as the training set size decreases. Notably, even when overfitted to smaller datasets, the model fails to replicate training samples sufficiently to recover diversity. Similarly, Latent Diffusion Models (LDMs) trained on the FFHQ dataset exhibit a comparable decline in diversity. To provide a comprehensive evaluation, we supplement the Vendi Score with RKE, Recall, and Coverage.

\begin{figure}[t]
    \centering
    \includegraphics[width=\linewidth]{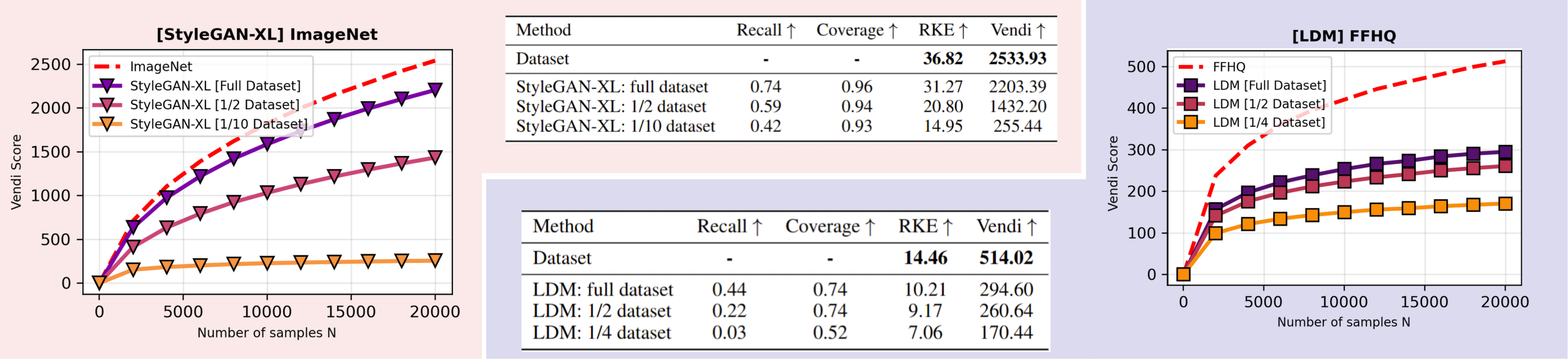}
    \caption{Vendi Score comparison across varying training set sizes. Results are shown for StyleGAN-XL (shaded red) and LDMs (shaded blue), with the training dataset baseline indicated by a dashed red line. Corresponding RKE, Recall, and Coverage metrics are also presented.}
    \label{fig:stylegan-xl-train}\vspace{-2mm}
\end{figure}
\begin{figure}
    \centering
    \includegraphics[width=0.95\linewidth]{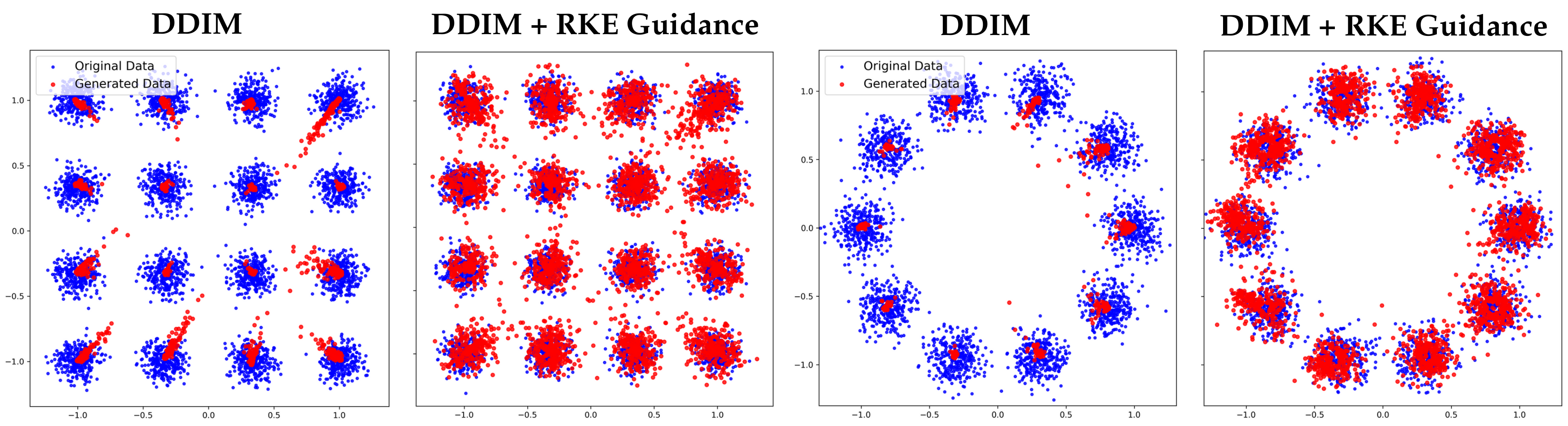}
    \caption{Comparison of the diffusion model (DDIM) with RKE regularized DDIM on an underlying 2D-Gaussian mixture model.}    \label{fig:gaussian_16_ddim}\vspace{-5mm}
\end{figure}

\begin{figure}
    \centering
    \includegraphics[width=0.98\linewidth]{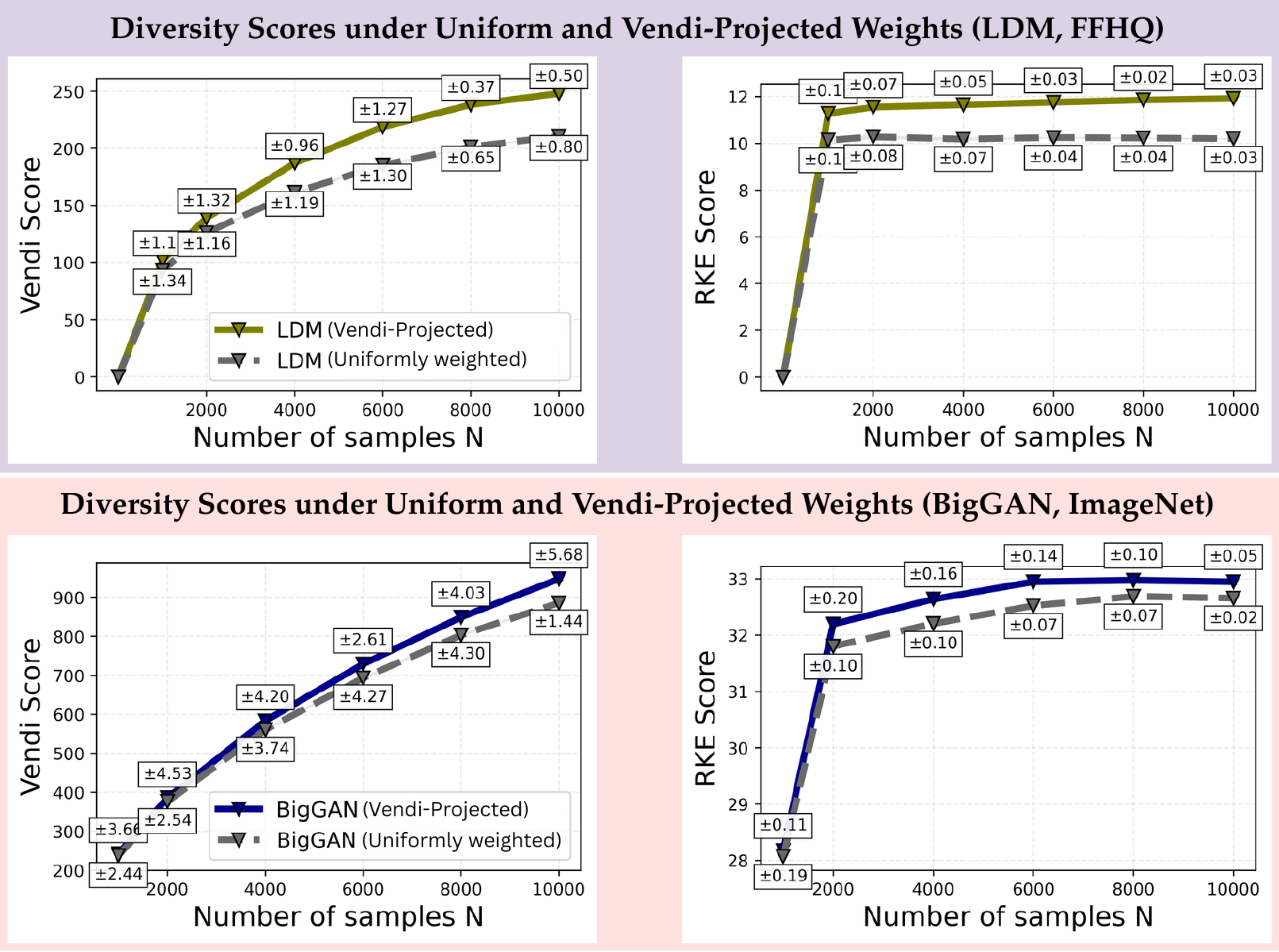}
    \caption{Vendi and RKE diversity scores vs. size $N$ for $N$ uniformly weighted generated samples and their post-hoc Vendi-projected reweighting. (Upper) LDM trained on FFHQ. (Lower) BigGAN trained on ImageNet. The Vendi-based projection yields higher diversity scores for the reweighted empirical distribution.
    }    \label{fig:weighted_samples_vendi_rke}\vspace{-3mm}
\end{figure}

\paragraph{Entropy-based guidance for generative models.} 
We evaluated the impact of RKE and Vendi regularization on generative models, demonstrating that guidance can effectively mitigate diversity biases while improving FD and KD scores. As shown in Figure~\ref{fig:intro_fig}, applying inverse-RKE guidance to the DiT-XL-2 diffusion model \cite{Peebles2022DiT} enhances both diversity and fidelity metrics. We observe similar results when guiding the DDIM model \cite{song2020denoising} on 2D Gaussian mixtures (Figure \ref{fig:gaussian_16_ddim}); increases in Vendi, RKE, and Recall confirm improved diversity, while better FD, KD, and Precision scores indicate a closer alignment with the dataset distribution. Extended quantitative results are provided in the Appendix.

\paragraph{Post-hoc Projected-Vendi reweighting.} To evaluate the effectiveness of the empirical-support projection mechanism, we apply the reweighting procedure in Equation~(\ref{eq:main_posthoc_projection_vendi_lagrangian}) (penalized variant with $\lambda=0.01$) to fixed sample sets of sizes ranging from 1000 to 10000 drawn from several state-of-the-art generative models on FFHQ and ImageNet. For each model and subsample size, we generate one large pool of samples and then create two empirical measures: (i) the uniform baseline (Uniform sampling), and (ii) the Projected-Vendi reweighted distribution obtained by solving the quadratic program. Diversity is measured via Vendi and RKE scores, while distributional proximity to the distribution of test samples is quantified via kernel distance (KD) and Fréchet distance (FD) in DINOv2 feature space. In Figure~\ref{fig:weighted_samples_vendi_rke}, we observe the Vendi and RKE trends as the number of samples increases. We find that diversity scores consistently improve across different sample sizes. The quantitative results, reported in Table~\ref{table:ffhq_imagenet} for ProjectedGAN \cite{sauer2021projected-gan}, BigGAN \cite{brock_large_2018}, and LDM \cite{rombach2022ldm}, demonstrate that Projected-Vendi consistently achieves considerably higher Vendi and RKE scores compared to uniform weighting, while incurring only modest increases in KD and FD, thereby confirming the projection's ability to reduce diversity bias.

\begin{table}
    \centering
    \caption{Evaluated scores for comparison of Uniform sampling vs. Vendi-Projected sampling.}   \label{table:ffhq_imagenet}    
    \begin{tabular}{cclrrrr}
    \toprule
    & Model & Sampling & Vendi $\uparrow$ & RKE $\uparrow$ & KD $\downarrow$ & FD $\downarrow$ \\
    \midrule
    \multirow{6}{*}{\rotatebox[origin=c]{90}{\textbf{FFHQ}}}
    & \raisebox{-2.5ex}{ProjectedGAN} & Uniform          & 193.64 & 9.31 & 1.958 & 586.43 \\
    &              & Projected-Vendi   & 237.34 & 10.88 & 1.642 & 562.92 \\
    \cmidrule(lr){2-7}
    & \raisebox{-2.5ex}{LDM}          & Uniform          & 210.32 & 10.21 & 0.587 & 223.32 \\
    &              & Projected-Vendi   & 248.01 & 11.92 & 0.462 & 213.87 \\
    \midrule
    \multirow{6}{*}{\rotatebox[origin=c]{90}{\textbf{ImageNet}}}
    & \raisebox{-2.5ex}{BigGAN}       & Uniform          & 885.68 & 32.56 &  0.539 & 403.57 \\
    &              & Projected-Vendi   & 947.32 & 33.27 & 0.517 & 386.34 \\
    \cmidrule(lr){2-7}
    & \raisebox{-2.5ex}{LDM}          & Uniform          & 889.44 & 33.83 & 0.102 & 124.31 \\
    &              & Projected-Vendi   & 919.18 & 34.89 & 0.094 & 105.95 \\
    \bottomrule
    \end{tabular}
\end{table}

\vspace{-1mm}
\section{Conclusion}\vspace{-1mm}
We studied entropy-based diversity in generative models using reference-free metrics and showed that finite-sample entropy exhibits a systematic downward bias relative to the underlying data distribution. Empirically, real datasets consistently achieve higher Vendi/RKE diversity than samples from standard pretrained generative models under matched evaluation protocols, and increasing training-set size improves model diversity. These results suggest that finite-sample effects and training-data limitations can contribute to observed diversity gaps. Our numerical studies rely mainly on controlled synthetic settings and finite-sample evaluations, and while entropy-based regularization and guidance are theoretically motivated, we do not implement exact entropy projection operators due to computational and optimization challenges. Extending these ideas to scalable implementations and broader benchmarks remains a direction for future work.\vspace{-3mm}

\clearpage
\bibliographystyle{plainnat}
\bibliography{ref}

 \begin{appendices}
 \section{Proofs}
\subsection{Proof of Proposition \ref{prop:monotone_logvendi}}

The von Neumann entropy is concave on the set of positive semidefinite, unit-trace operators (set of density matrices).
Fix $n\ge 1$ and draw $X_1,\ldots,X_{n+1}\sim P$ i.i.d.
Let $Z_i:=\phi(X_i)\phi(X_i)^\top$.
Consider a uniformly random subset $S\subset\{1,\ldots,n+1\}$ of size $n$, independent of the data, and define
$\widehat{C}_S=\frac{1}{n}\sum_{i\in S} Z_i$.
Conditioned on $(Z_1,\ldots,Z_{n+1})$, we have
$\mathbb{E}[\widehat{C}_S]=\widehat{C}_{n+1}$.
Therefore, by the concavity of the entropy function, we can apply Jensen's inequality to obtain
\[
H(\widehat{C}_{n+1})
\ge
\mathbb{E}\left[H(\widehat{C}_S)\,\big|\, Z_1,\ldots,Z_{n+1}\right].
\]
Taking expectation and noting that $\widehat{C}_S$ has the same marginal distribution as $\widehat{C}_n$ completes the proof.

\subsection{Proof of concentration of $\log$-$\mathrm{Vendi}$ score around its finite-sample expectation}

Here, we show that a single-sample estimate of log-Vendi (the von Neumann entropy of the empirical
kernel covariance) concentrates around its expectation under i.i.d.\ sampling. This result provides justification on
reporting averaged Vendi/RKE scores over multiple random subsets at a fixed evaluation size.

To review the notation in the result, let $k:\mathcal{X}\times\mathcal{X}\to\mathbb{R}$ be a normalized bounded kernel with
$k(x,x)=1$ and $|k(x,x')|\le 1$ for all $x,x'$.
Let $\phi:\mathcal{X}\to\mathcal{H}$ be a feature map into a real Hilbert space such that
$k(x,x')=\langle \phi(x),\phi(x')\rangle_{\mathcal{H}}$.
For $m\ge 2$ and i.i.d.\ samples $X_{1:m}\sim P$, define the empirical kernel covariance operator
\[
\widehat{C}_m:=\frac{1}{m}\sum_{i=1}^m \phi(X_i)\phi(X_i)^\top,
\]
and the corresponding log-Vendi and Vendi scores
\[
\mathsf{Ent}_m := H(\widehat{C}_m),
\qquad
\mathsf{Vendi}_m := \exp(\mathsf{Ent}_m),
\qquad
H(\rho):=-\mathrm{Tr}(\rho\log\rho).
\]
Note that we define the stability constant $c_m$ as follows
\begin{equation}\label{eq:cm_def_app}
c_m := \frac{1}{m}\log(2m-1)+h\left(\frac{1}{m}\right),
\qquad
h(t):=-t\log t-(1-t)\log(1-t).
\end{equation}
As we show in the theorem's proof, it can be seen that $c_m\le \frac{\log(em)}{m}$ holds for every integer $m$.
\begin{proposition}
\label{prop:vendi_concentration_single_eval_app}
Fix $m\ge 2$ and let $X_{1:m}\overset{\mathrm{i.i.d.}}{\sim}P$.
Then for any $\delta\in(0,1)$, with probability at least $1-\delta$,
\begin{equation}\label{eq:ent_conc_app}
\big|\mathsf{Ent}_m-\mathbb{E}[\mathsf{Ent}_m]\big|
\le
c_m\sqrt{\frac{m}{2}\log\left(\frac{2}{\delta}\right)} \le \sqrt{\frac{\log^2(em)\log(2/\delta)}{2m}}.
\end{equation}
Equivalently, letting $t_\delta:=\sqrt{\frac{\log^2(em)\log(2/\delta)}{2m}}$, with probability at least $1-\delta$, the following holds
\begin{equation}\label{eq:vendi_mult_app}
e^{\mathbb{E}[\mathsf{Ent}_m]-t_\delta}
\le
\mathsf{Vendi}_m
\le
e^{\mathbb{E}[\mathsf{Ent}_m]+t_\delta}.
\end{equation}
\end{proposition}

\begin{proof}
To apply McDiarmid's inequality,
we define the function
\[
f(x_1,\ldots,x_m)
:=
H\left(\frac{1}{m}\sum_{i=1}^m \phi(x_i)\phi(x_i)^\top\right).
\]
We will show that replacing any single input changes $f$ by at most $c_m$.

Fix an index $j\in[m]$ and consider two input sequences that differ only at the $j$-th coordinate:
$(x_1,\ldots,x_j,\ldots,x_m)$ and $(x_1,\ldots,x_j',\ldots,x_m)$.
Let $\widehat{C}$ and $\widehat{C}'$ be the corresponding empirical covariance operators.
Write $u=\phi(x_j)$ and $v=\phi(x_j')$.
Since $k$ is normalized, $\|u\|_{\mathcal{H}}^2=k(x_j,x_j)=1$ and similarly $\|v\|_{\mathcal{H}}=1$.
Then
\[
\widehat{C}-\widehat{C}'
=
\frac{1}{m}\big(uu^\top - vv^\top\big).
\]
Each rank-one operator has trace norm
$\|uu^\top\|_1=\mathrm{Tr}(uu^\top)=\|u\|_{\mathcal{H}}^2=1$ (and similarly for $vv^\top$),
hence by the triangle inequality
\[
\|uu^\top - vv^\top\|_1 \le \|uu^\top\|_1+\|vv^\top\|_1 = 2
\quad\Longrightarrow\quad
\|\widehat{C}-\widehat{C}'\|_1 \le \frac{2}{m}.
\]
Let
\[
\Delta := \frac{1}{2}\|\widehat{C}-\widehat{C}'\|_1 \le \frac{1}{m}.
\]
Both $\widehat{C}$ and $\widehat{C}'$ are PSD with unit trace and rank at most $m$,
since each is an average of $m$ rank-one PSD operators.
Let $d$ be the dimension of the subspace spanned by $\mathrm{supp}(\widehat{C})\cup \mathrm{supp}(\widehat{C}')$.
Then
\[
d \le \mathrm{rank}(\widehat{C})+\mathrm{rank}(\widehat{C}') \le 2m.
\]
Von Neumann entropy depends only on the nonzero spectrum, thus we may restrict both operators to this $d$-dimensional subspace.
Applying Audenaert's continuity bound for von Neumann entropy \cite{audenaert2007sharp} results in the following inequality:
\begin{equation}\label{eq:audenaert_bound_app}
\bigl\vert H(\widehat{C})-H(\widehat{C}')\bigr\vert
\le
\Delta\log(d-1) + h(\Delta).
\end{equation}
Using $d\le 2m$ and $\Delta\le 1/m$, and the facts that $\log(d-1)$ is increasing in $d$ and
$h(\Delta)$ is increasing for $\Delta\in[0,1/2]$ (here $\Delta\le 1/m\le 1/2$ since $m\ge 2$), we obtain
\[
|H(\widehat{C})-H(\widehat{C}')|
\le
\frac{1}{m}\log(2m-1) + h\left(\frac{1}{m}\right)
=
c_m.
\]

Next, we apply McDiarmid's inequality to the defined function. 
Let $\mathsf{Ent}_m = f(X_1,\ldots,X_m)$.
By McDiarmid's inequality, for any $t>0$,
\[
\mathbb{P}\left(\big|\mathsf{Ent}_m-\mathbb{E}[\mathsf{Ent}_m]\big|\ge t\right)
\le
2\exp\left(-\frac{2t^2}{\sum_{j=1}^m c_j^2}\right)
=
2\exp\left(-\frac{2t^2}{m c_m^2}\right).
\]
Setting the right-hand side equal to $\delta$ yields
$t = c_m\sqrt{\frac{m}{2}\log\left(\frac{2}{\delta}\right)}$,
which proves \eqref{eq:ent_conc_app}.

To convert log-Vendi concentration to Vendi concentration, we note that
the bound \eqref{eq:ent_conc_app} implies
$\mathbb{E}[\mathsf{Ent}_m]-t_\delta \le \mathsf{Ent}_m \le \mathbb{E}[\mathsf{Ent}_m]+t_\delta$
with probability at least $1-\delta$.
Exponentiating all terms gives \eqref{eq:vendi_mult_app}.

Also, for a simple upper bound on $h(1/m)$, we note that for $0<p\le 1/2$, a standard inequality is $h(p)\le p\log(e/p)$.
Taking $p=1/m$ (valid for $m\ge 2$) gives
\[
h\bigl(\frac{1}{m}\bigr) \le \frac{1}{m}\log(em).
\]
The proof is therefore complete.
\end{proof}

\subsection{Proof of Theorem~\ref{thm:entropy_projection_main}}
\label{app:entropy_projection_proofs}
We first prove the following lemma on the convexity of the entropy's super-level sets and then prove Parts (i) and (ii) separately. 

\begin{lemma}\label{lem:superlevel_convex}
Let $\mathcal{P}$ be a convex set and let $H:\mathcal{P}\to\mathbb{R}$ be concave.
For any $\rho\in\mathbb{R}$, the set
\[
\mathcal{C}_\rho := \{Q\in\mathcal{P}: H(Q)\ge \rho\}
\]
is  a convex set.
\end{lemma}

\begin{proof}
Take any $Q_1,Q_2\in\mathcal{C}_\rho$ and any $\alpha\in[0,1]$.
Since $\mathcal{P}$ is convex, $Q_\alpha:=\alpha Q_1+(1-\alpha)Q_2\in\mathcal{P}$.
By concavity of $H$,
\[
H(Q_\alpha)\ \ge\ \alpha H(Q_1)+(1-\alpha)H(Q_2)\ \ge\ \alpha\rho+(1-\alpha)\rho\ =\ \rho
\]
Therefore, $Q_\alpha\in\mathcal{C}_\rho$.
\end{proof}


\subsubsection*{(i) Hilbertian distances}
We first review the definitions of Hilbertian distance and Hilbertian distance projection. 

We call $\mathrm{dist}$  \emph{Hilbertian} on $\mathcal{P}$ if there exist a Hilbert space
$\mathcal{H}$ and an affine map $\mu:\mathcal{P}\to\mathcal{H}$ such that for all $P,Q\in\mathcal{P}$,
\[
\mathrm{dist}(P,Q)=\|\mu(P)-\mu(Q)\|_{\mathcal{H}}.
\]

Also, for the Hilbertian distance projection, we assume $\mathcal{C}_\rho$ is nonempty and convex, and that $\mu(\mathcal{C}_\rho)$ is closed in
$\mathcal{H}$ (equivalently, $\mathcal{C}_\rho$ is closed under $\mathrm{dist}$).
Then for every $Q_\theta\in\mathcal{P}$, Hilbertian distance projection is
\[
Q^\star\in\arg\min_{Q\in\mathcal{C}_\rho}\mathrm{dist}(Q_\theta,Q).
\]

\begin{lemma}[Pythagorean inequality for Hilbert projections]
\label{lem:hilbert_pythag}
Let $C\subseteq\mathcal{H}$ be nonempty, closed, and convex.
Let $x^\star=\Pi_C(x)$ denote the metric projection of $x$ onto $C$.
Then for every $y\in C$,
\[
\|y-x\|_{\mathcal{H}}^2\ \ge\ \|y-x^\star\|_{\mathcal{H}}^2+\|x-x^\star\|_{\mathcal{H}}^2.
\]
\end{lemma}

\begin{proof}
It can be seen a metric projection in Hilbert spaces satisfies the following:
$x^\star=\Pi_C(x)$ if and only if
\[
\langle x-x^\star,\ y-x^\star\rangle_{\mathcal{H}}\le 0\quad\text{for all }y\in C.
\]
We then expand the difference norm squared as follows
\[
\|y-x\|^2
=\|y-x^\star\|^2+\|x-x^\star\|^2+2\langle y-x^\star,\ x^\star-x\rangle.
\]
The last inner product is $\le 0$ by the projection characterization, and thus the inequality follows.
\end{proof}

To derive \eqref{eq:hilbert_contract_main}, we let $C=\mu(\mathcal{C}_\rho)\subseteq\mathcal{H}$, $x=\mu(Q_\theta)$,
$x^\star=\mu(Q^\star)$, and $y=\mu(P)$ (note $P\in\mathcal{C}_\rho$ implies $y\in C$).
Applying Lemma~\ref{lem:hilbert_pythag} results in the following
\[
\|y-x\|^2\ge \|y-x^\star\|^2+\|x-x^\star\|^2,
\]
and hence in particular $\|y-x^\star\|\le \|y-x\|$.
Translating back using $\mathrm{dist}(P,Q)=\|\mu(P)-\mu(Q)\|_{\mathcal{H}}$ yields
\[
\mathrm{dist}(P,Q^\star)\le \mathrm{dist}(P,Q_\theta).
\]

\subsubsection*{(ii) Bregman divergences}
We first review the definition of Bregman divergence and I-Projection.

Let $\Omega$ be a convex set in a real vector space and let $\Phi:\Omega\to\mathbb{R}$ be differentiable
and strictly convex.
The associated Bregman divergence is
\[
D_\Phi(p,q):=\Phi(p)-\Phi(q)-\langle \nabla\Phi(q),\,p-q\rangle.
\]

Regarding the I-projection,
we assume $\mathcal{C}_\rho$ is nonempty and convex, $\mathcal{C}_\rho\subseteq \Omega$,
and that $\mathcal{C}_\rho$ is closed and the function $p\mapsto D_\Phi(p,Q_\theta)$ is coercive on
$\mathcal{C}_\rho$ (for example, if $\mathcal{C}_\rho$ is compact).
Then, the I-Projection follows from
\[
Q^\star\in\arg\min_{Q\in\mathcal{C}_\rho}D_\Phi(Q,Q_\theta).
\]

\begin{lemma}
\label{lem:three_point}
For any $p,u,q\in\Omega$,  the following hold
\[
D_\Phi(p,q)=D_\Phi(p,u)+D_\Phi(u,q)+\langle \nabla\Phi(u)-\nabla\Phi(q),\,p-u\rangle.
\]
\end{lemma}

\begin{proof}
We expand each $D_\Phi$ term using the definition as follows:
\begin{align*}
D_\Phi(p,u)+D_\Phi(u,q)
&=\Phi(p)-\Phi(u)-\langle\nabla\Phi(u),p-u\rangle
+\Phi(u)-\Phi(q)-\langle\nabla\Phi(q),u-q\rangle\\
&=\Phi(p)-\Phi(q)-\langle\nabla\Phi(q),p-q\rangle
+\langle\nabla\Phi(q)-\nabla\Phi(u),p-u\rangle\\
&=D_\Phi(p,q)+\langle\nabla\Phi(q)-\nabla\Phi(u),p-u\rangle,
\end{align*}
which is equivalent to the stated identity.
\end{proof}

\begin{lemma}[Bregman Pythagorean inequality for I-projections]
\label{lem:bregman_pythag}
Let $C\subseteq\Omega$ be nonempty and convex, and let
$u\in\arg\min_{v\in C}D_\Phi(v,q)$ be an I-projection of $q$ onto $C$.
Then for every $p\in C$,
\[
D_\Phi(p,q)\ \ge\ D_\Phi(p,u)+D_\Phi(u,q).
\]
\end{lemma}

\begin{proof}
We first establish the first-order optimality condition:
since $v\mapsto D_\Phi(v,q)$ is convex and differentiable with gradient
$\nabla_v D_\Phi(v,q)=\nabla\Phi(v)-\nabla\Phi(q)$,
optimality of $u$ over the convex set $C$ implies
\begin{equation}
\label{eq:bregman_opt_cond}
\langle \nabla\Phi(u)-\nabla\Phi(q),\,p-u\rangle \ge 0
\qquad\text{for all }p\in C.
\end{equation}
Then, we apply Lemma~\ref{lem:three_point} with $(p,u,q)$ to obtain
\[
D_\Phi(p,q)=D_\Phi(p,u)+D_\Phi(u,q)+\langle \nabla\Phi(u)-\nabla\Phi(q),\,p-u\rangle.
\]
Using \eqref{eq:bregman_opt_cond}, the inner product term is nonnegative, hence
$D_\Phi(p,q)\ge D_\Phi(p,u)+D_\Phi(u,q)$.
\end{proof}

To derive \eqref{eq:bregman_contract_main},
we apply Lemma~\ref{lem:bregman_pythag} with $C=\mathcal{C}_\rho$, $q=Q_\theta$, and $u=Q^\star$.
For any $P\in\mathcal{C}_\rho$,
\[
D_\Phi(P,Q_\theta)\ \ge\ D_\Phi(P,Q^\star)+D_\Phi(Q^\star,Q_\theta)\ \ge\ D_\Phi(P,Q^\star),
\]
which gives $D_\Phi(P,Q^\star)\le D_\Phi(P,Q_\theta)$, i.e. \eqref{eq:bregman_contract_main}.

Finally, note that on the probability simplex, forward KL divergence $\mathrm{KL}(p\|q)$ is the Bregman divergence generated by
$\Phi(p)=\sum_i p_i\log p_i$ (negative Shannon entropy).
Therefore the above I-projection guarantee applies to the KL-divergence.
\qed

\section{Empirical-support entropy projection: VNE and RKE}
\label{app:empirical_support_projection}

This appendix provides the technical details for the empirical-support post-hoc projection described in Section~\ref{sec:practical_projection}. The goal is to reweight a fixed set of $M$ generated samples to satisfy a target diversity level while staying close (in a Hilbertian geometry) to the baseline uniform empirical distribution.

\subsection{Setup: reweighted empirical measures and kernel geometry}
Let $\{x_i\}_{i=1}^M\subset\mathcal{X}$ be samples generated from a trained model distribution $Q_\theta$. For any weight vector $q\in\Delta_M:=\{q\ge 0,\ \sum_{i=1}^M q_i=1\}$ define the reweighted empirical measure
\[
Q_q := \sum_{i=1}^M q_i\,\delta_{x_i},
\qquad
q_0:=\tfrac{1}{M}\mathbf{1}.
\]
Let $k$ be a positive semidefinite kernel with Gram matrix $K\in\mathbb{R}^{M\times M}$, $K_{ij}=k(x_i,x_j)$.

\paragraph{Squared MMD/KD on a fixed support.}
For two weighted empirical measures $Q_q$ and $Q_{q'}$ supported on the same points $\{x_i\}$, the squared MMD (equivalently KD in our notation) with kernel $k$ satisfies
\begin{equation}\label{eq:app_kd_quad}
\mathrm{KD}^2(Q_q,Q_{q'})=\sum_{i,j}(q_i-q'_i)(q_j-q'_j)k(x_i,x_j)=(q-q')^\top K (q-q').
\end{equation}
In particular, $\mathrm{KD}^2(Q_q,Q_{q_0})=(q-q_0)^\top K(q-q_0)$ is a convex quadratic function of $q$.

\subsection{VNE (log-Vendi) on empirical support}

Let $k(x,y)=\langle\phi(x),\phi(y)\rangle$ be represented by a feature map $\phi$ into a (possibly infinite-dimensional) Hilbert space. Define the empirical-support covariance operator
\begin{equation}\label{eq:app_cov_op}
C(q):=\sum_{i=1}^M q_i\,\phi(x_i)\phi(x_i)^\top.
\end{equation}
Assume the kernel is normalized on the support: $k(x_i,x_i)=\|\phi(x_i)\|^2=1$ for all $i$. Then $\mathrm{tr}(C(q))=\sum_i q_i=1$.

\paragraph{Empirical-support VNE.}
Define the von Neumann entropy
\begin{equation}\label{eq:app_vne_def}
H_{\mathrm{VNE}}(q):=-\mathrm{tr}\!\bigl(C(q)\log C(q)\bigr),
\end{equation}
with the convention $0\log 0:=0$.

\begin{proposition}[Concavity of empirical-support VNE in $q$]\label{prop:app_vne_concavity}
Under the above normalization, $q\mapsto H_{\mathrm{VNE}}(q)$ is concave on $\Delta_M$. Consequently, for any $\rho\in\mathbb{R}$, the following superlevel set $\mathcal{C}^{\mathrm{VNE}}_\rho$ is a convex set:
\[
\mathcal{C}^{\mathrm{VNE}}_\rho := \{q\in\Delta_M:\ H_{\mathrm{VNE}}(q)\ge \rho\}
\]
\end{proposition}
\begin{proof}
To prove the statement, let $\mathcal{D}:=\{X\succeq 0:\mathrm{tr}(X)=1\}$ and define $S(X):=-\mathrm{tr}(X\log X)$ on $\mathcal{D}$.
We show $S$ is concave on $\mathcal{D}$ via a variational identity.

For any Hermitian $H$, define $\sigma_H := e^{H}/\mathrm{tr}(e^{H})$. For $X\in\mathcal{D}$, the quantum relative entropy
$D(X\|\sigma_H):=\mathrm{tr}(X(\log X-\log\sigma_H))$ satisfies $D(X\|\sigma_H)\ge 0$.
Using $\log\sigma_H = H - \log\mathrm{tr}(e^{H})$, we obtain
\[
D(X\|\sigma_H)=\mathrm{tr}(X\log X)-\mathrm{tr}(XH)+\log\mathrm{tr}(e^{H})\ge 0,
\]
hence $S(X)\le \log\mathrm{tr}(e^{H})-\mathrm{tr}(XH)$ for all $H$. Equality is achieved by taking $\sigma_H=X$
(equivalently $H=\log X$ on the support of $X$), yielding
\[
S(X)=\inf_{H=H^\top}\bigl\{\log\mathrm{tr}(e^{H})-\mathrm{tr}(XH)\bigr\}.
\]
For each fixed $H$, the map $X\mapsto \log\mathrm{tr}(e^{H})-\mathrm{tr}(XH)$ is affine; an infimum of affine functions is concave, hence $S$ is concave on $\mathcal{D}$.

Since $q\mapsto C(q)$ is affine by \eqref{eq:app_cov_op} and $C(q)\in\mathcal{D}$ for all $q\in\Delta_M$ by normalization, the composition $q\mapsto S(C(q))=H_{\mathrm{VNE}}(q)$ is concave. Convexity of the superlevel set follows.
\end{proof}

\paragraph{Gradient with respect to $q$.}
Differentiating through $C(q)$ gives the Fr\'echet differential identity
\begin{equation}\label{eq:app_vne_diff}
\mathrm{d}H_{\mathrm{VNE}}(q)
=
-\mathrm{tr}\!\bigl((\log C(q)+I)\,\mathrm{d}C(q)\bigr),
\end{equation}
and since $\partial C(q)/\partial q_i=\phi(x_i)\phi(x_i)^\top$,
\begin{equation}\label{eq:app_vne_grad}
\frac{\partial}{\partial q_i}H_{\mathrm{VNE}}(q)
=
-\bigl\langle \phi(x_i),(\log C(q)+I)\phi(x_i)\bigr\rangle,
\qquad i\in[M].
\end{equation}
In practice, $\log C(q)$ is computed on the support of $C(q)$ (equivalently using a small spectral floor), which is standard for stable entropy computations.

\paragraph{Finite-dimensional evaluation via weighted Gram matrices.}
Let $K=\Phi\Phi^\top$ for a (possibly implicit) feature matrix $\Phi$ with rows $\phi(x_i)^\top$. Define
\begin{equation}\label{eq:app_Aq}
A(q):=\mathrm{diag}(\sqrt{q})\,K\,\mathrm{diag}(\sqrt{q}).
\end{equation}
Then $A(q)$ and $C(q)$ share the same nonzero eigenvalues, so $H_{\mathrm{VNE}}(q)$ can be computed from an eigendecomposition of $A(q)$, and \eqref{eq:app_vne_grad} can be evaluated using the corresponding eigenbasis.

\subsection{Empirical-support VNE projection problem and solver}

We define the empirical-support VNE projection as
\begin{equation}\label{eq:app_vne_proj}
\min_{q\in\Delta_M}\ (q-q_0)^\top K (q-q_0)
\qquad\text{s.t.}\qquad
H_{\mathrm{VNE}}(q)\ge \rho.
\end{equation}
By Proposition~\ref{prop:app_vne_concavity} and PSD-ness of $K$, this is a convex optimization problem over the simplex.

A practical solver is primal--dual KL-mirror descent (exponentiated gradient) on the Lagrangian
\[
\mathcal{F}(q,\lambda)=(q-q_0)^\top K(q-q_0)+\lambda(\rho-H_{\mathrm{VNE}}(q)),\qquad \lambda\ge 0.
\]
The primal gradient is
\[
\nabla_q\mathcal{F}(q,\lambda)=2K(q-q_0)-\lambda\nabla_q H_{\mathrm{VNE}}(q),
\]
with $\nabla_q H_{\mathrm{VNE}}(q)$ from \eqref{eq:app_vne_grad}.

\begin{algorithm}[t]
\caption{Empirical-support VNE projection (primal--dual exponentiated gradient)}
\label{alg:app_vne_proj}
\KwIn{$K$; target $\rho$; stepsizes $\eta,\gamma>0$; iterations $T$}
\KwOut{$q_T\in\Delta_M$}
Initialize $q_0\leftarrow \frac{1}{M}\mathbf{1}$ and $\lambda_0\leftarrow 0$\;
\For{$t=0,1,\dots,T-1$}{
Compute $H_{\mathrm{VNE}}(q_t)$ and $\nabla_q H_{\mathrm{VNE}}(q_t)$ using an eigendecomposition of $A(q_t)=\mathrm{diag}(\sqrt{q_t})K\mathrm{diag}(\sqrt{q_t})$\;
$g_t \leftarrow 2K(q_t-q_0)-\lambda_t\,\nabla_q H_{\mathrm{VNE}}(q_t)$\;
$\tilde q_{t+1,i}\leftarrow q_{t,i}\exp(-\eta\,g_{t,i})$ for all $i$;\quad $q_{t+1}\leftarrow \tilde q_{t+1}/\sum_j \tilde q_{t+1,j}$\;
$\lambda_{t+1}\leftarrow \bigl[\lambda_t+\gamma(\rho-H_{\mathrm{VNE}}(q_{t+1}))\bigr]_+$\;
}
\Return{$q_T$}\;
\end{algorithm}

\subsection{Order-2 (RKE) empirical-support projection}

Define the squared-kernel Gram matrix
\[
\widetilde K_{ij}:=k(x_i,x_j)^2,
\qquad
\widetilde K:=K\circ K,
\]
where $\circ$ denotes the Hadamard product. For $Q_q=\sum_i q_i\delta_{x_i}$,
\begin{equation}\label{eq:app_irke_q}
\mathrm{Inverse\text{-}RKE}(Q_q)
=
\mathbb{E}_{X,X'\sim Q_q}[k(X,X')^2]
=
\sum_{i,j}q_i q_j k(x_i,x_j)^2
=
q^\top \widetilde K q.
\end{equation}
Since $K\succeq 0$ implies $\widetilde K=K\circ K\succeq 0$ (Schur product theorem), $q\mapsto q^\top \widetilde K q$ is convex on $\Delta_M$.

A convenient penalized formulation is the convex quadratic program
\begin{equation}\label{eq:app_rke_pen}
\min_{q\in\Delta_M}\ (q-q_0)^\top K(q-q_0)+\lambda\,q^\top \widetilde K q,
\qquad \lambda>0.
\end{equation}
The gradient is explicit:
\[
\nabla_q\Bigl((q-q_0)^\top K(q-q_0)+\lambda\,q^\top \widetilde K q\Bigr)=2K(q-q_0)+2\lambda \widetilde K q,
\]
so exponentiated-gradient updates yield a simple solver.

\begin{algorithm}[t]
\caption{Empirical-support order-2 projection (exponentiated gradient)}
\label{alg:app_rke_proj}
\KwIn{$K$; penalty $\lambda>0$; stepsize $\eta>0$; iterations $T$}
\KwOut{$q_T\in\Delta_M$}
Compute $\widetilde K\leftarrow K\circ K$ and initialize $q_0\leftarrow \frac{1}{M}\mathbf{1}$\;
\For{$t=0,1,\dots,T-1$}{
$g_t \leftarrow 2K(q_t-q_0)+2\lambda\,\widetilde K q_t$\;
$\tilde q_{t+1,i}\leftarrow q_{t,i}\exp(-\eta\,g_{t,i})$ for all $i$;\quad
$q_{t+1}\leftarrow \tilde q_{t+1}/\sum_j \tilde q_{t+1,j}$\;
}
\Return{$q_T$}\;
\end{algorithm}

\section{Additional Numerical Results}

\begin{figure}
    \centering
    \includegraphics[width=\linewidth]{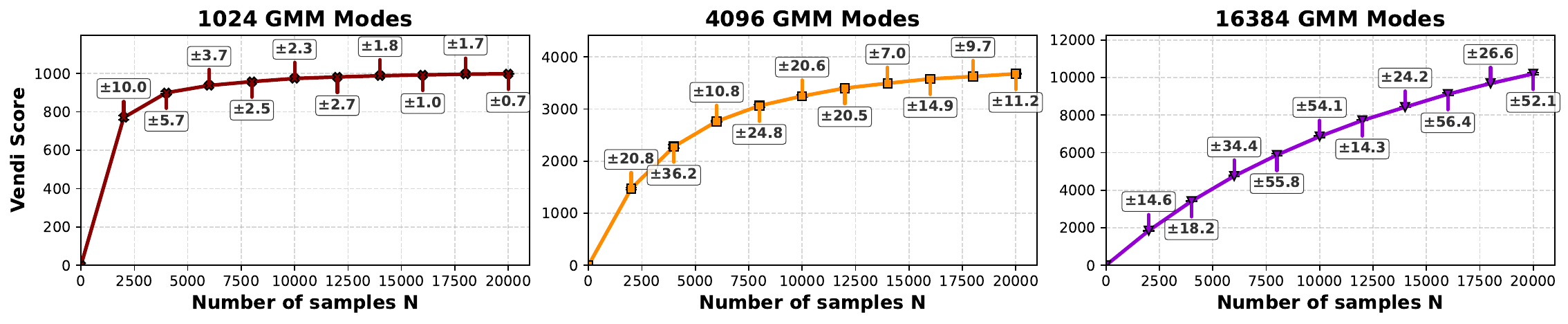}
    \caption{Vendi Score convergence for $d$-dimensional Gaussian Mixtures. The score is evaluated across varying sample sizes $n$ for $2^d$ isotropic Gaussians ($\sigma_{\text{std}} = 10^{-4}$) centered on the vertices of a $d$-dimensional cube. We set Gaussian kernel bandwidth $\sigma$ to 0.1. Error bars represent 95\% confidence intervals calculated over 10 independent trials.}
    \label{fig:gaussian mixture bias}
\end{figure}
\subsection{Extended analysis of Vendi score convergence on empirical distributions}

Building on the results in Figure~\ref{fig:ci-dinov2-sigma35}, we extend our analysis to AFHQ-v2~\cite{choi2020starganv2} and LSUN-Bedroom~\cite{yu2015lsun}, using parameters identical to those in the main text. Note that for AFHQ-v2, which contains only 15,803 images, we do not report a confidence interval for the final data point. We observe that regardless of the chosen dataset, the increasing trend remains. 

\begin{figure}[ht]
    \centering
    \includegraphics[width=\textwidth]{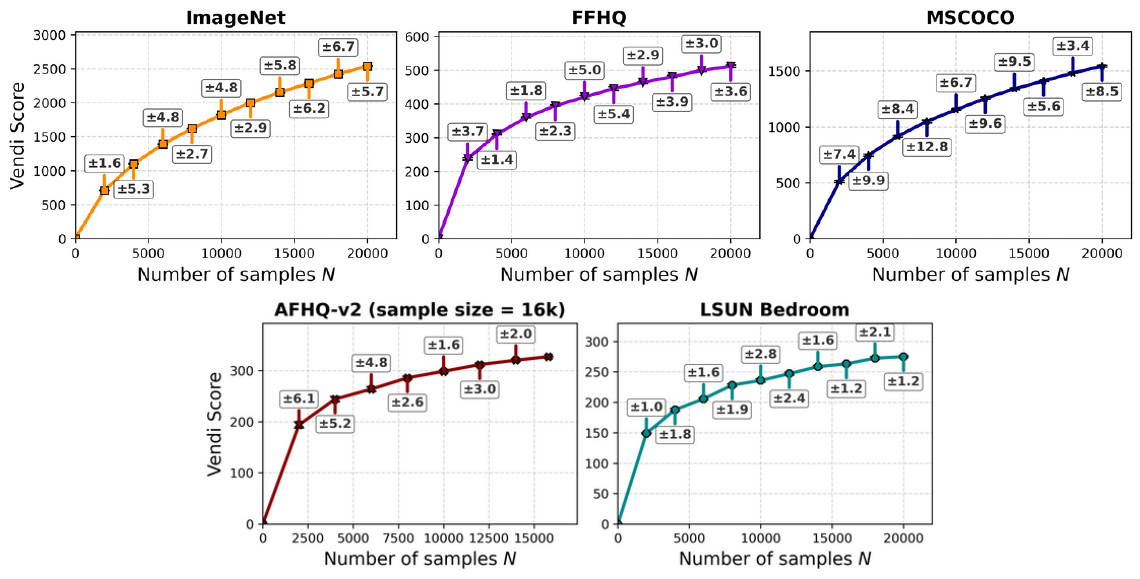}
    \caption{Vendi score curves (mean and 95\% confidence intervals over 5 independent sample sets) for ImageNet, FFHQ, MSCOCO, LSUN Bedroom (sample size values of $n\le 20\text{K}$) and AFHQ-v2 (sample size values $n \leq 16000$), computed using DINOv2 embeddings and a Gaussian (RBF) kernel with bandwidth $\sigma=35$.}
    \label{fig:all-ci-datasets}
\end{figure}

\subsection{Extended analysis of generative models}

In this section we report more results on comparing generative models to the underlying training set. We provide results for a variety of architectures across different embedding models.

\subsubsection{The effect of chosen embeddings on ranking of generative models}

Extending the comparison in Figure~\ref{fig:models-comparison-2x2}, we evaluated the generative models across four distinct embedding architectures: DINOv2~\cite{oquab_dinov2_2023}, DINOv3~\cite{siméoni2025dinov3}, CLIP~\cite{radford_learning_2021}, and InceptionV3~\cite{szegedy2016inception}. We observe that for all embeddings, the diversity of the real dataset consistently matches or exceeds that of the generated samples. These results are detailed in Figures~\ref{fig:dinov2-variety-datasets}--\ref{fig:dinov3-variety-datasets}.

\begin{figure}[htb]
    \centering
    \subfigure[ImageNet]{\includegraphics[width=0.85\textwidth]{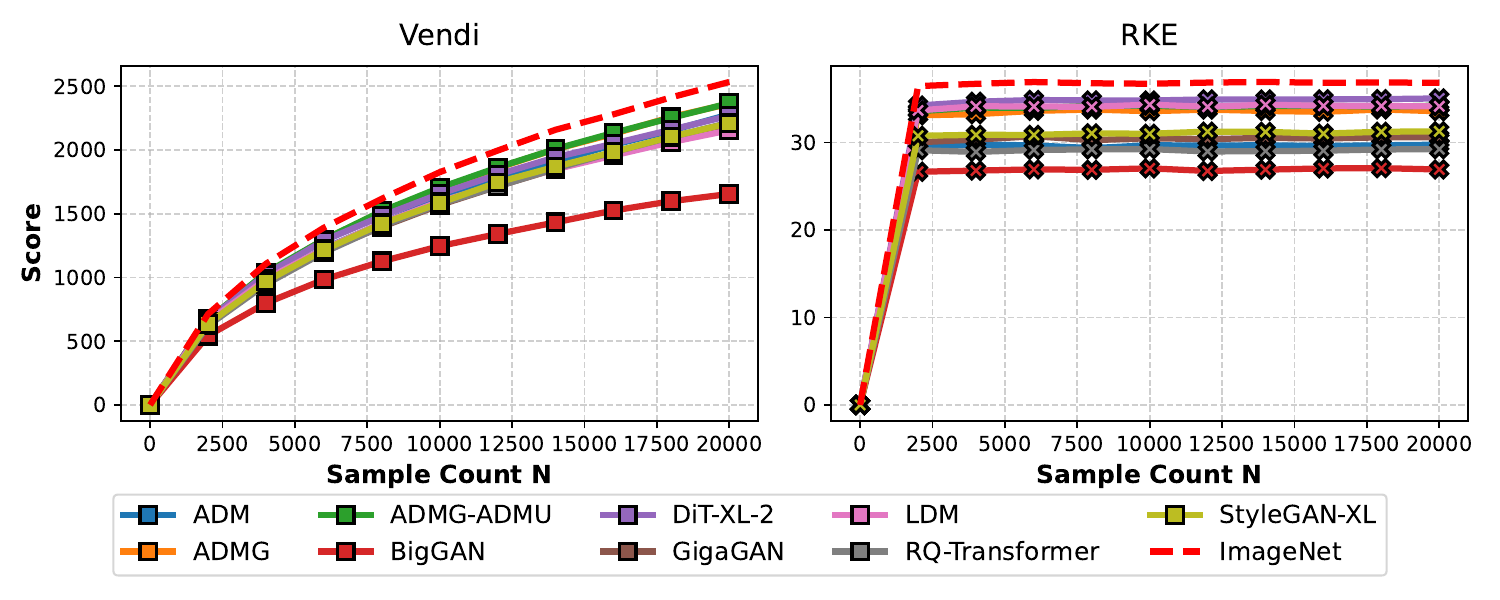}}
    \subfigure[FFHQ]{\includegraphics[width=0.85\textwidth]{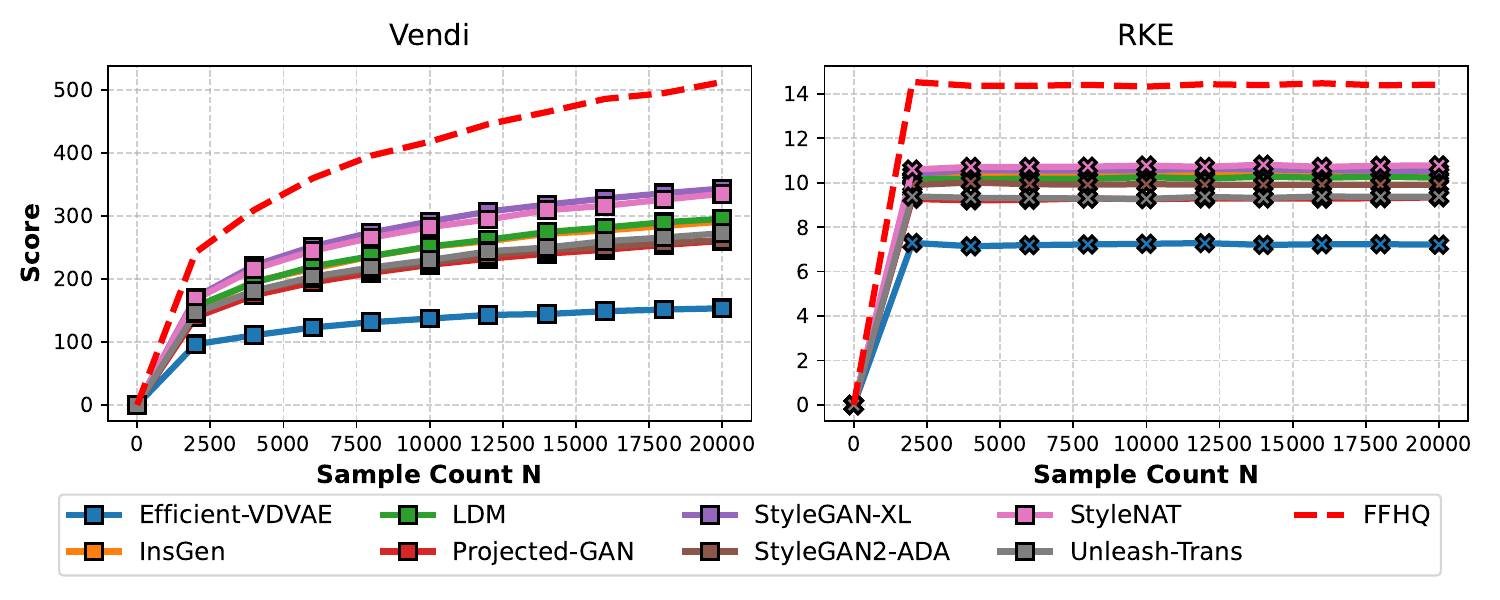}}
    \subfigure[LSUN Bedroom]{\includegraphics[width=0.85\textwidth]{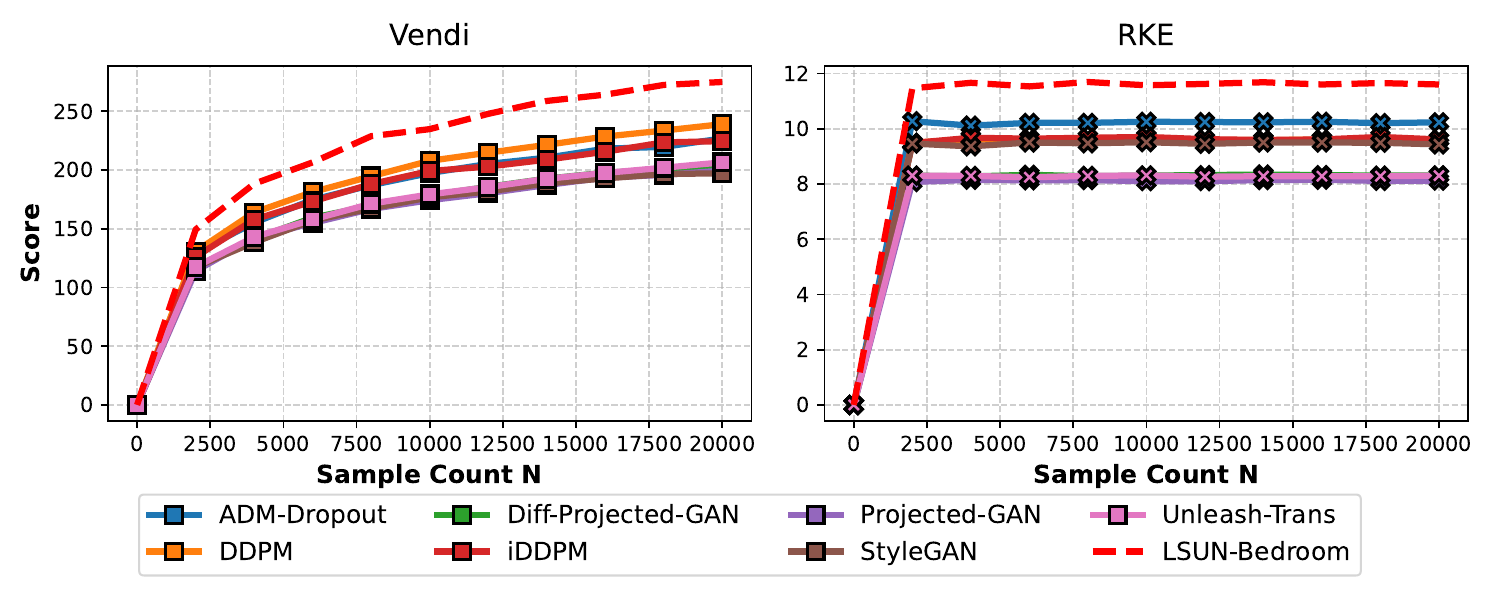}}
    \caption{Comparison of Vendi scores of the test sample set (the dashed red curve) and the generated samples by pre-trained generative models across three datasets. The backbone embedding is DINOv2 embeddings using Gaussian kernel with bandwidth $\sigma=35$.}
    \label{fig:dinov2-variety-datasets}
\end{figure}

\begin{figure}[ht]
    \centering
    \subfigure[ImageNet]{\includegraphics[width=0.85\textwidth]{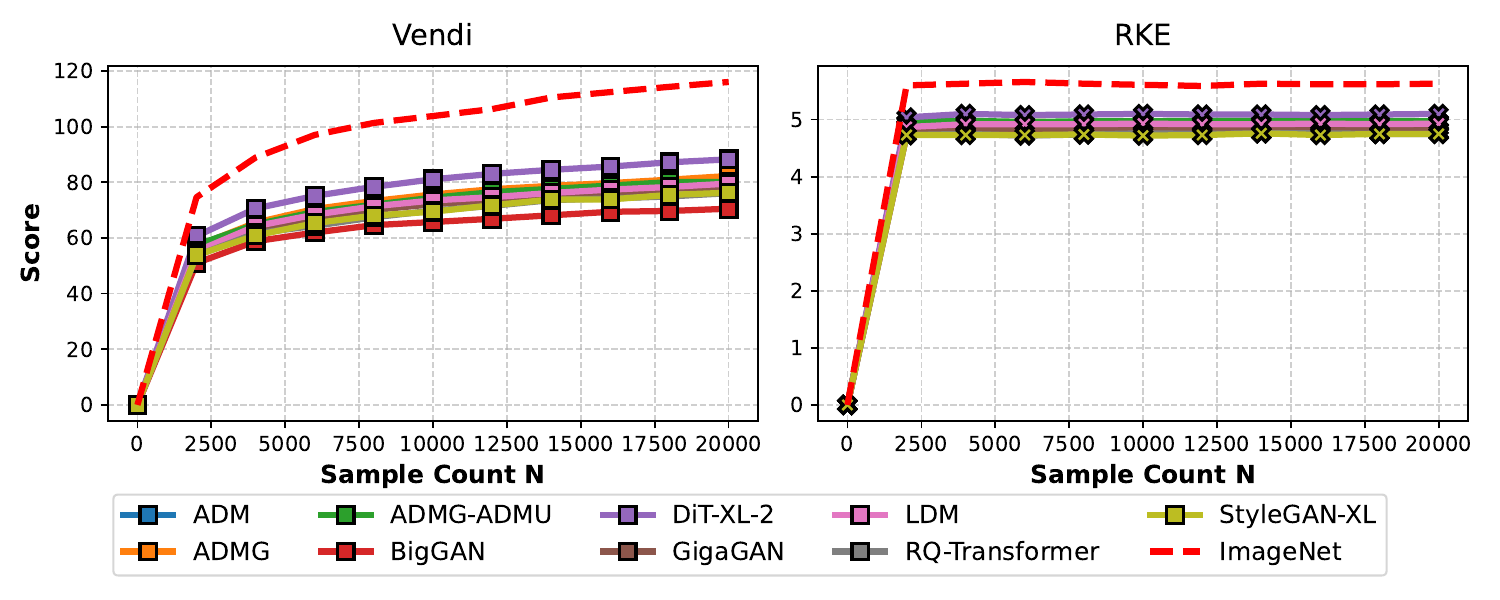}}
    \subfigure[FFHQ]{\includegraphics[width=0.85\textwidth]{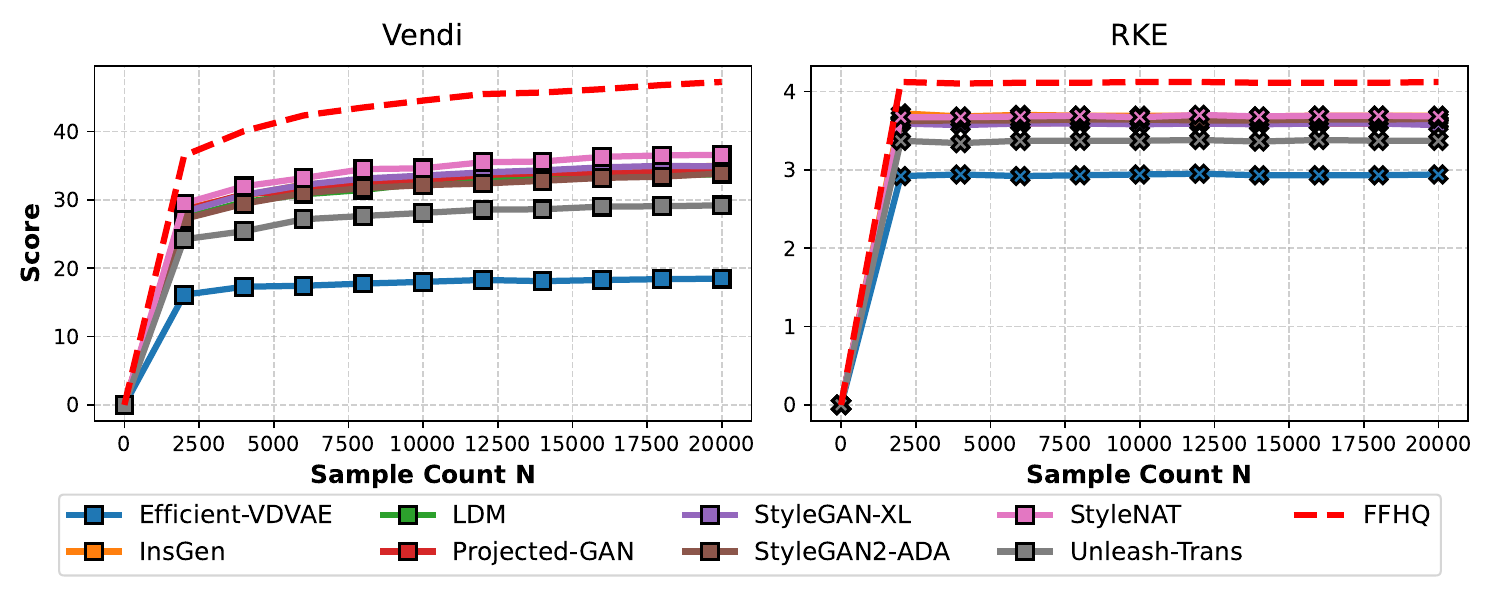}}
    \subfigure[LSUN Bedroom]{\includegraphics[width=0.85\textwidth]{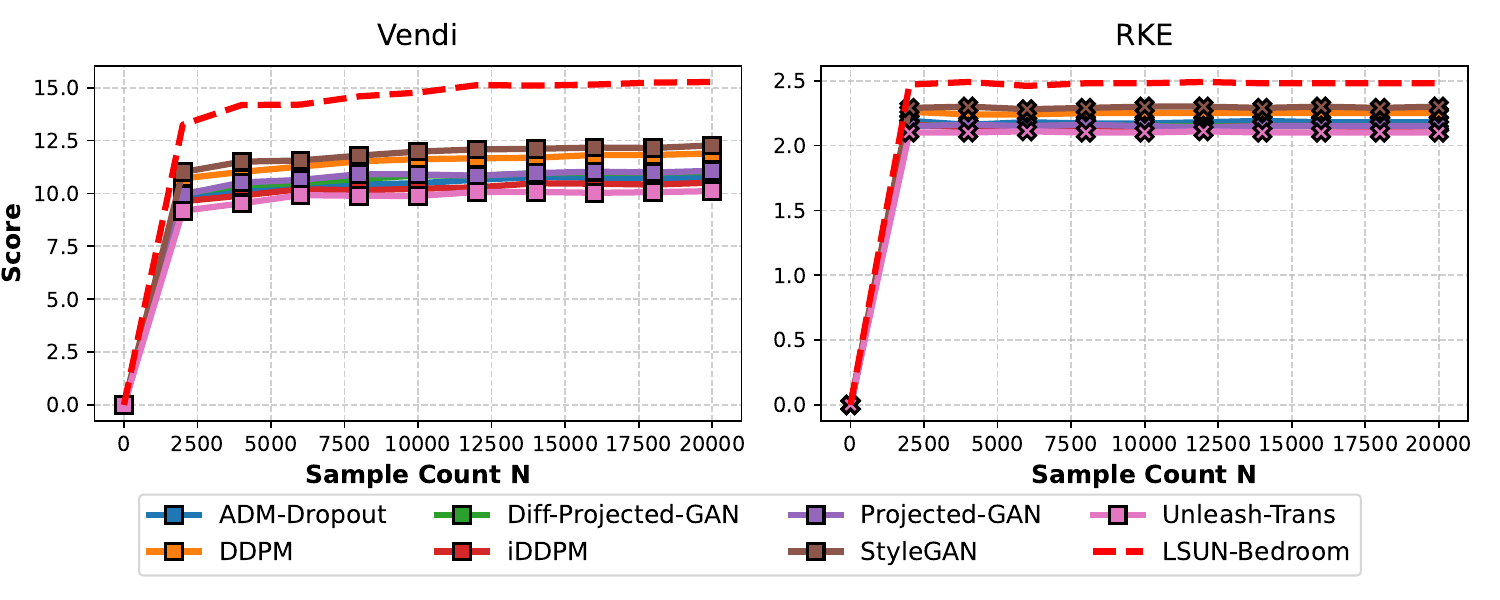}}
    \caption{Comparison of Vendi scores of the test sample set (the dashed red curve) and the generated samples by pre-trained generative models across three datasets. The backbone embedding is CLIP embeddings using Gaussian kernel with bandwidth $\sigma=8.0$.}
    \label{fig:clip-variety-datasets}
\end{figure}

\begin{figure}[ht]
    \centering
    \subfigure[ImageNet]{\includegraphics[width=0.85\textwidth]{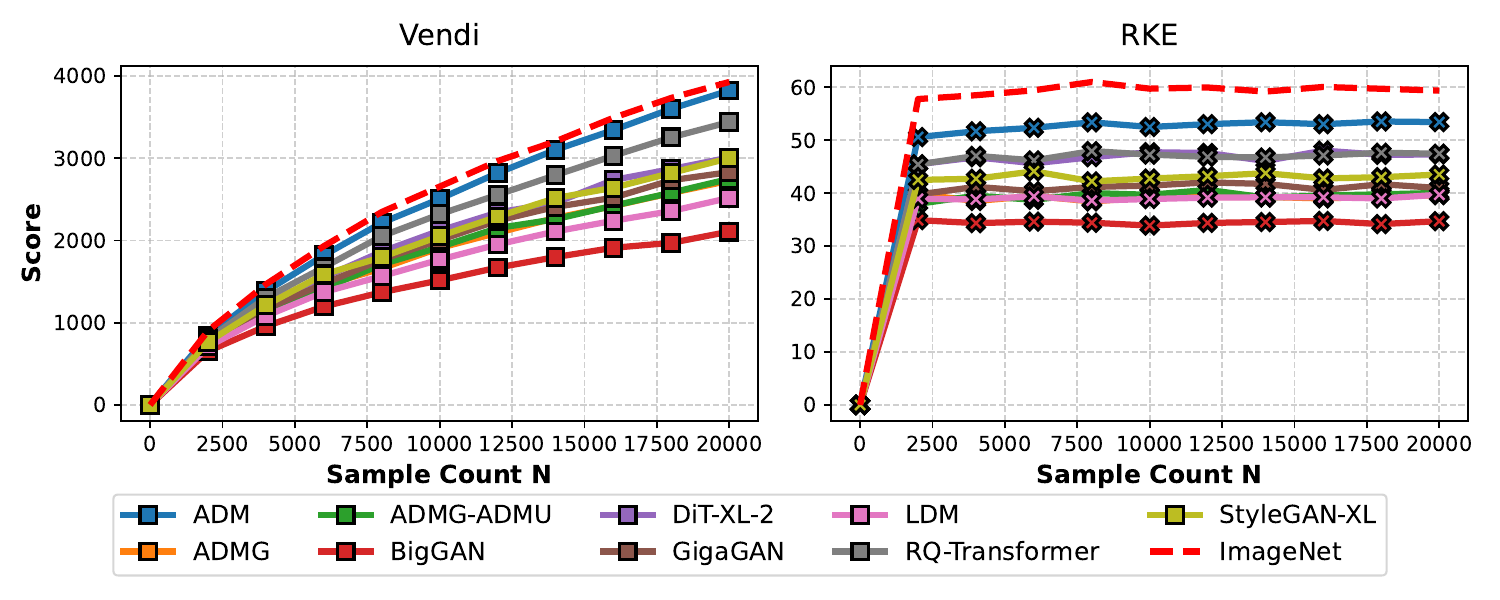}}
    \subfigure[FFHQ]{\includegraphics[width=0.85\textwidth]{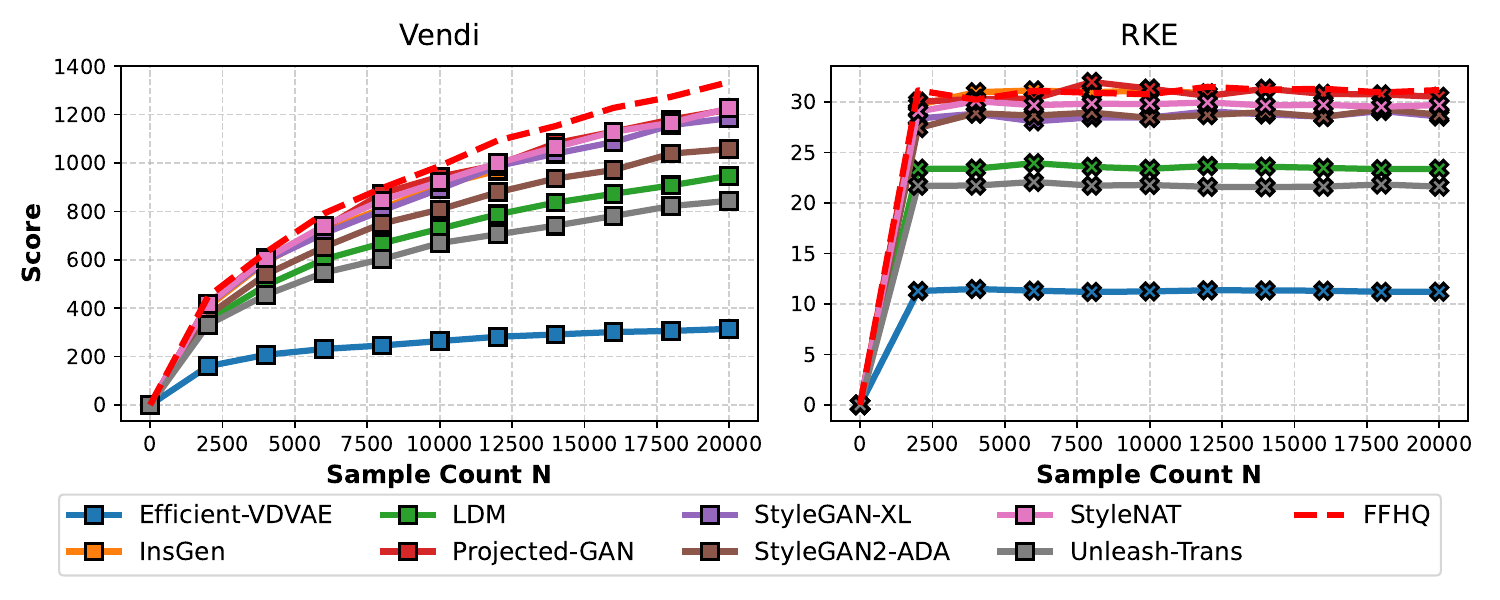}}
    \subfigure[LSUN Bedroom]{\includegraphics[width=0.85\textwidth]{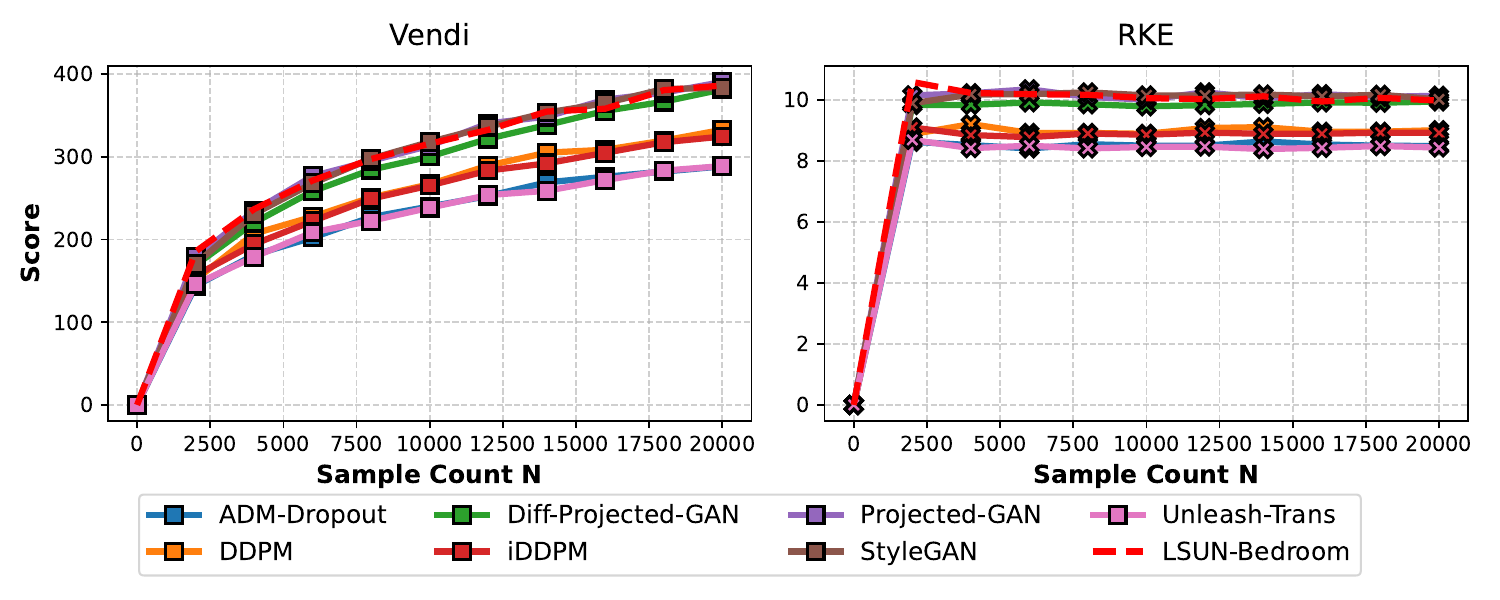}}
    \caption{Comparison of Vendi scores of the test sample set (the dashed red curve) and the generated samples by pre-trained generative models across three datasets. The backbone embedding is InceptionV3 embeddings using Gaussian kernel with bandwidth $\sigma=9.0$.}
    \label{fig:inception-variety-datasets}
\end{figure}

\begin{figure}[ht]
    \centering
    \subfigure[ImageNet]{\includegraphics[width=0.85\textwidth]{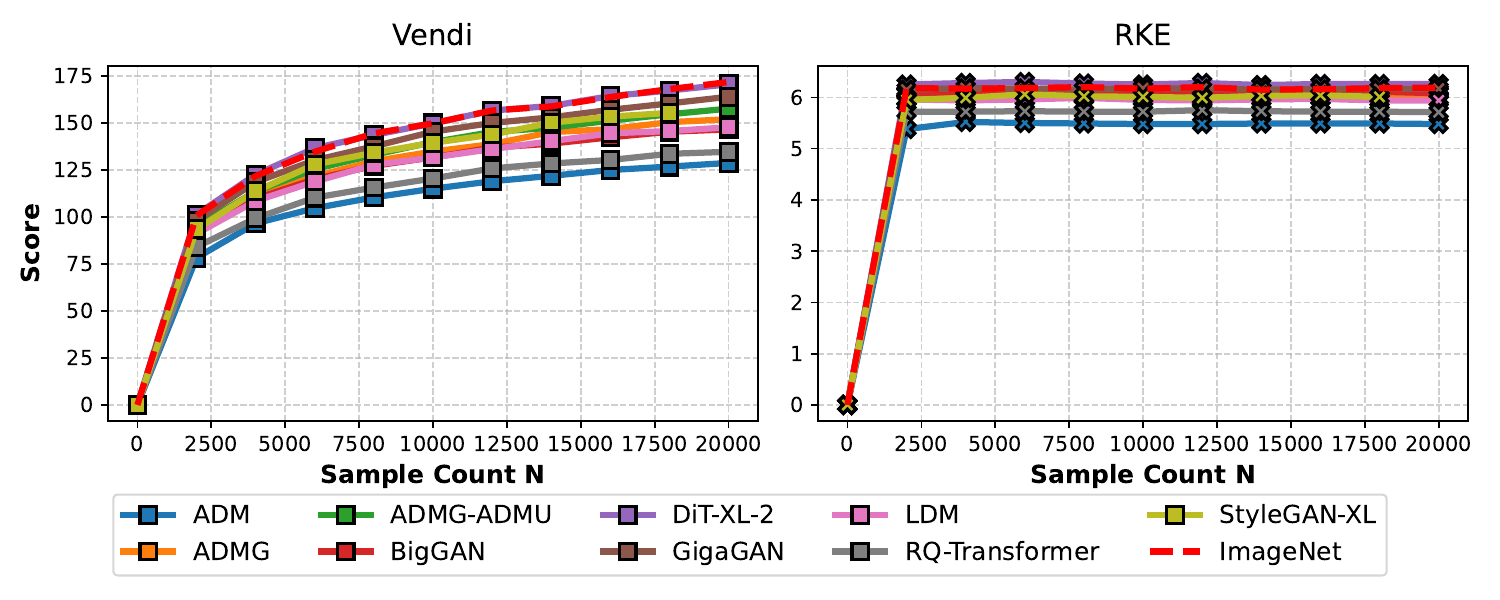}}
    \subfigure[FFHQ]{\includegraphics[width=0.85\textwidth]{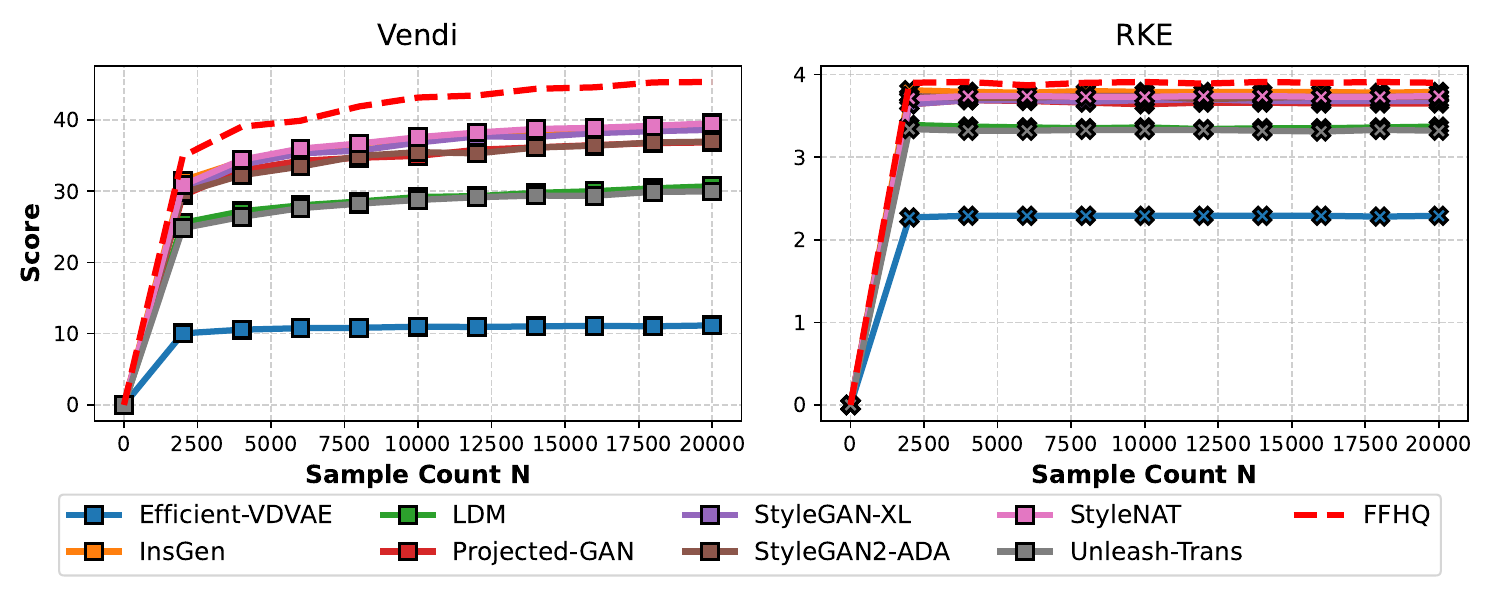}}
    \subfigure[LSUN Bedroom]{\includegraphics[width=0.85\textwidth]{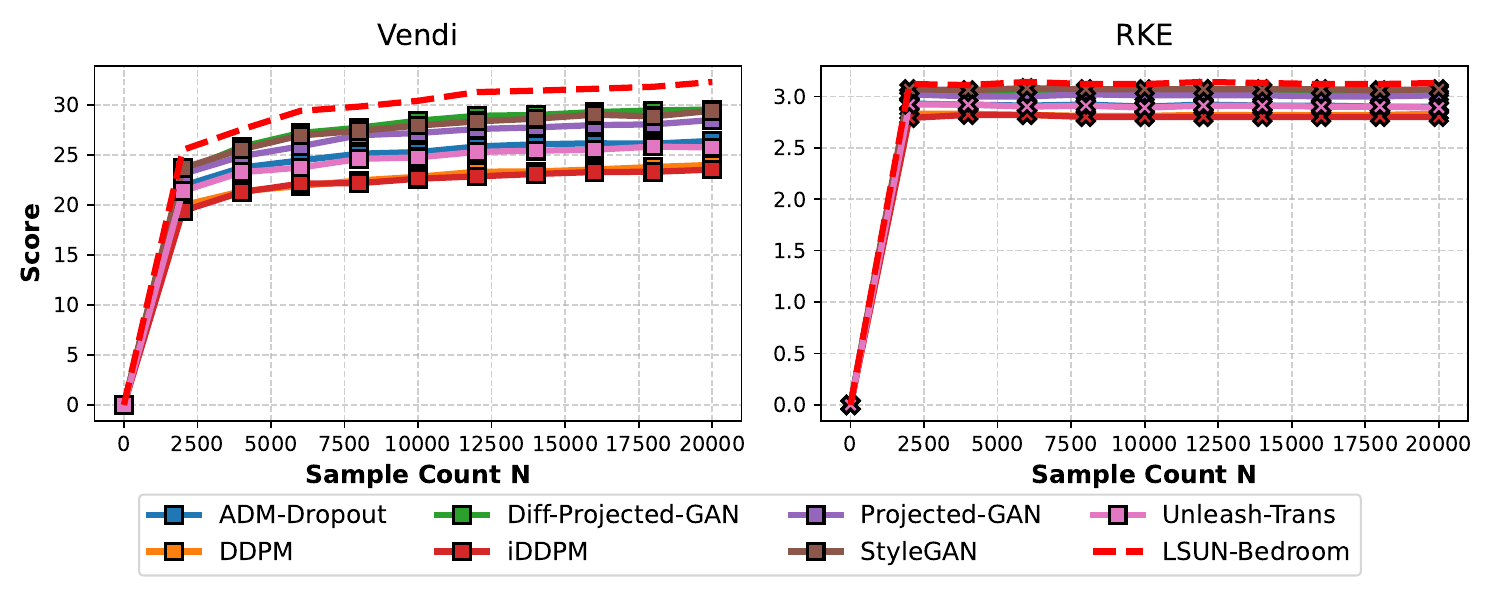}}
    \caption{Comparison of Vendi scores of the test sample set (the dashed red curve) and the generated samples by pre-trained generative models across three datasets. The backbone embedding is DINOv3-ViTS16 embeddings using Gaussian kernel with bandwidth $\sigma=8.0$.}
    \label{fig:dinov3-variety-datasets}
\end{figure}

\subsubsection{The effect of choosing a cosine similarity kernel}
To provide a comprehensive analysis of the Vendi metric, we also evaluated diversity using the cosine similarity kernel. As established by \citet{ospanov2025vendi}, the convergence of this kernel scales with the dimension of the embedding model. We confirm that this behavior holds across all four datasets. As shown with Gaussian kernel the diversity of the training set consistently exceeds that of the generative models. The downward diversity bias of the generative holds even when we swap the underlying kernel function.

\begin{figure}[ht]
    \centering
    \includegraphics[width=\textwidth]{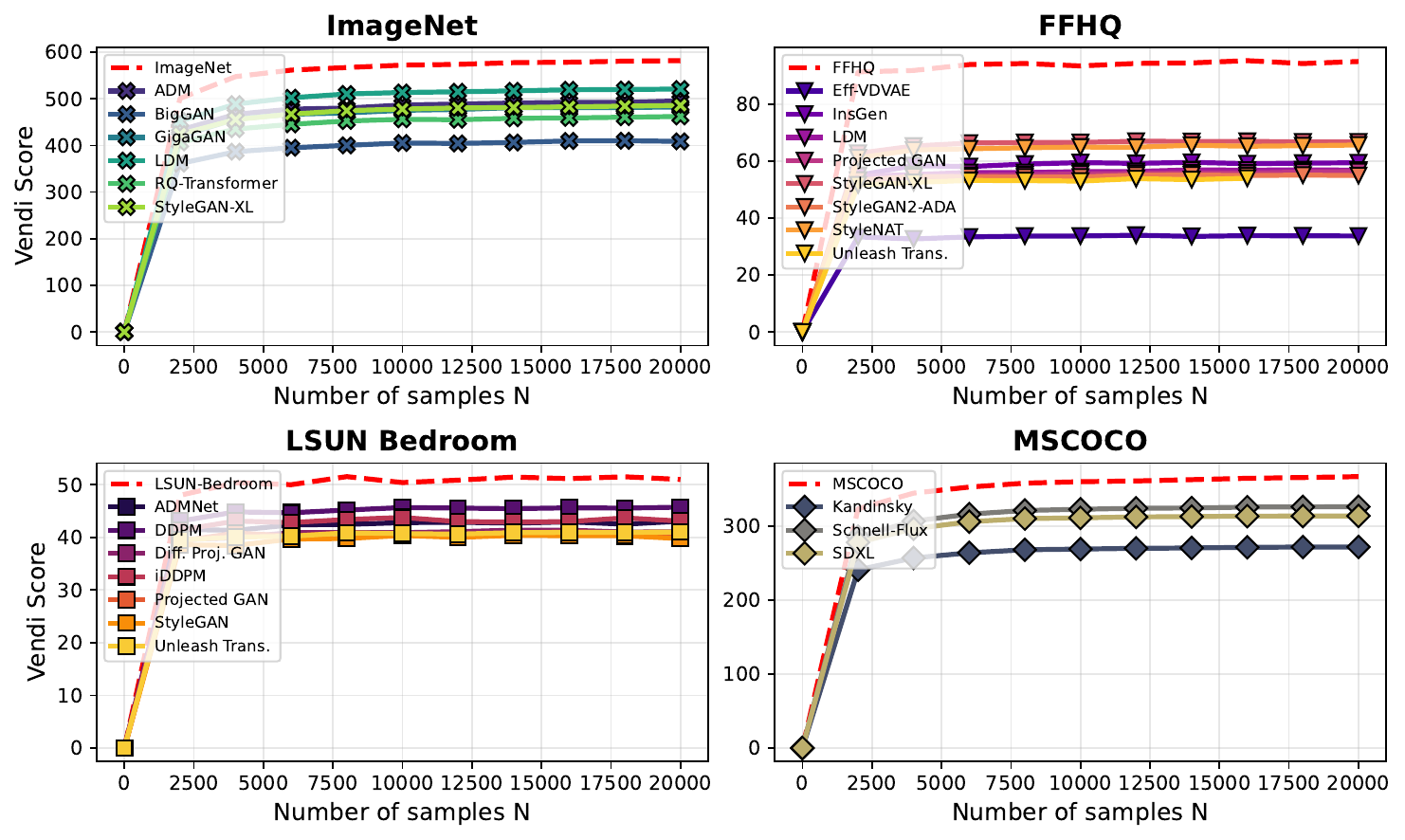}
    \caption{Comparison of generative model architectures across four datasets. We evaluate models on ImageNet, FFHQ, LSUN Bedroom, and MSCOCO using DINOv2 embeddings with cosine similarity kernel.}
    \label{fig:models-comparison-cosine-2x2}
\end{figure}

\subsubsection{Extending results to reference-based metrics (Recall and Coverage)}
To provide a comprehensive analysis of generative model diversity, we also report the reference-based Recall~\cite{kynkaanniemi_improved_2019} and Coverage~\cite{naeem2020reliable} metrics. As these metrics require a reference distribution, scores for the datasets themselves are omitted. For completeness, we present these results alongside the corresponding RKE and Vendi scores. We report scores in Tables~\ref{tab:imagenet_rke_vs_recall_coverage}, \ref{tab:ffhq_rke_vs_recall_coverage} and \ref{tab:lsun_rke_vs_recall_coverage}.

\begin{table}[ht]
\centering
\caption{Comparison of RKE, Coverage, and Recall across different model architectures trained on ImageNet. Coverage and Recall metrics were computed using \citet{stein_exposing_2023} implementation. Best model scores are underlined.}
\label{tab:imagenet_rke_vs_recall_coverage}
\begin{tabular}{lcccc}
\toprule
Method & Recall $\uparrow$ & Coverage $\uparrow$ & RKE $\uparrow$ & Vendi $\uparrow$\\
\midrule
Dataset & \textbf{-} & \textbf{-} & \textbf{36.82} & \textbf{2533.93}\\
\midrule
ADM \cite{dhariwal2021adm}  & \underline{0.79} & 0.89 & 29.74 & \underline{2279.73}\\
BigGAN \cite{brock_large_2018} & 0.44 & 0.57 & 26.91 & 1655.41\\
GigaGAN \cite{kang2023gigagan} & 0.74 & 0.70 & 30.59 & 2174.02\\
LDM \cite{rombach2022ldm} & 0.76 & 0.93 & \underline{34.13} & 2155.19\\
RQ-Transformer \cite{lee2022autoregressive} & 0.76 & 0.59 & 29.23 & 2214.47\\
StyleGAN-XL \cite{sauer2022styleganxl} & 0.74 & \underline{0.96} & 31.27 & 2209.83\\
\bottomrule
\end{tabular}
\end{table}

\begin{table}[ht]
\centering
\caption{Comparison of RKE, Coverage, and Recall across different model architectures trained on FFHQ. Coverage and Recall metrics were computed using \citet{stein_exposing_2023} implementation. Best model scores are underlined.}
\label{tab:ffhq_rke_vs_recall_coverage}
\begin{tabular}{lcccc}
\toprule
Method & Recall $\uparrow$ & Coverage $\uparrow$ & RKE $\uparrow$ & Vendi $\uparrow$\\
\midrule
Dataset & \textbf{-} & \textbf{-} & \textbf{14.46} & \textbf{512.84}\\
\midrule
LDM~\cite{rombach2022ldm}  & \underline{0.44} & \underline{0.74} & 10.21 & 295.47\\
Eff. vdVAE \cite{hazami2022efficientvdvae} & 0.14 & 0.52 & 7.21 & 153.45\\
InsGen \cite{yang2021insgen} & 0.14 & 0.51 & 10.38 & 291.21\\
Projected GAN \cite{sauer2021projected-gan}& 0.07 & 0.30 & 9.34 & 260.41\\
StyleGAN XL \cite{sauer2022styleganxl}& 0.42 & 0.61 & 10.53 & \underline{343.63}\\
StyleGAN2-Ada \cite{karras2020stylegan2ada}& 0.04 & 0.41 & 9.90 & 261.83\\
StyleNAT \cite{walton2022stylenat}& 0.42 & 0.71 & \underline{10.78} & 335.26\\
Unleashing Trans. \cite{bondtaylor2022unleashing}& 0.24 & 0.54 & 9.37 & 272.79\\
\bottomrule
\end{tabular}
\end{table}

\begin{table}[ht]
\centering
\caption{Comparison of RKE, Coverage, and Recall across different model architectures trained on LSUN-Bedroom. Coverage and Recall metrics were computed using \citet{stein_exposing_2023} implementation. Dataset scores are bolded, best model scores are underlined.}
\label{tab:lsun_rke_vs_recall_coverage}
\begin{tabular}{lcccc}
\toprule
Method & Recall $\uparrow$ & Coverage $\uparrow$ & RKE $\uparrow$ & Vendi $\uparrow$ \\
\midrule
Dataset & \textbf{-} & \textbf{-} & \textbf{11.60} & \textbf{274.90} \\
\midrule
ADMNet Dropout \cite{dhariwal2021adm}& \underline{0.75} & \underline{0.90} & \underline{10.23} & 226.85\\
DDPM \cite{ho2020ddpm}& 0.61 & 0.68 & 9.61 & \underline{239.00}\\
Diffusion-Projected GAN \cite{wang2022diffprojgan}& 0.28 & 0.28 & 8.31 & 204.06 \\
iDDPM \cite{nichol2021iddpm}& 0.64 & 0.76 & 9.61 & 224.55\\
Projected GAN \cite{sauer2021projected-gan}& 0.22 & 0.24 & 8.12 & 199.45 \\
StyleGAN \cite{karras2019style}& 0.41 & 0.69 & 9.43 & 197.21\\
Unleasing Trans. \cite{bondtaylor2022unleashing}& 0.41 & 0.42 & 8.30 & 206.47 \\
\bottomrule
\end{tabular}
\end{table}

\goodbreak\newpage
\subsection{Extended analysis of generative models trained on subsets of training data}

\subsubsection{Additional results on LDM, StyleGAN-XL and UNet based architectures}
Expanding on the findings in Figure~\ref{fig:stylegan-xl-train}, we evaluated two additional architectures: a small U-Net and a Latent Diffusion Model (LDM). The U-Net was trained on 100\%, 10\%, 2\%, and 1\% subsets of ImageNet, while the LDM was trained on 100\%, 50\%, and 25\% subsets of FFHQ, using standard technical configurations for each. Figures~\ref{fig:dinov2-all-trained-models} (DINOv2) and \ref{fig:clip-all-trained-models} (CLIP) confirm that the downward diversity bias persists across architectures and amplifies as the training set size decreases.

\begin{figure}
    \centering
    \includegraphics[width=0.95\textwidth]{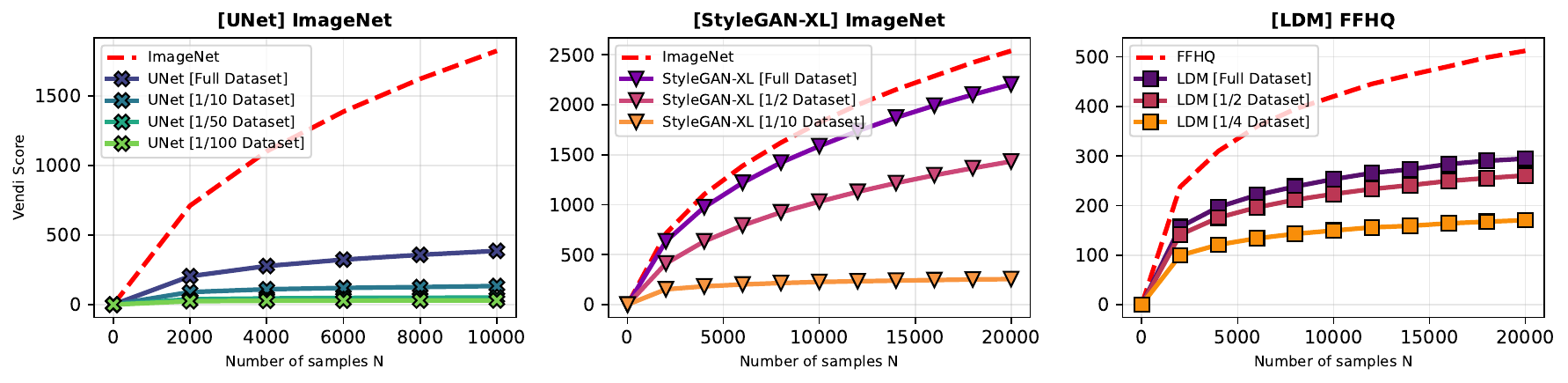}
    \caption{Vendi Score comparison for UNet, StyleGAN-XL, LDM. The models were trained on varying subsets of ImageNet, ImageNet, FFHQ respectively (indicated by fraction of the original size). Dashed red line represents score of the corresponding training dataset. The evaluation results are done using DINOv2 with $\sigma=35$.}
    \label{fig:dinov2-all-trained-models}
\end{figure}

\begin{figure}
    \centering
    \includegraphics[width=0.95\textwidth]{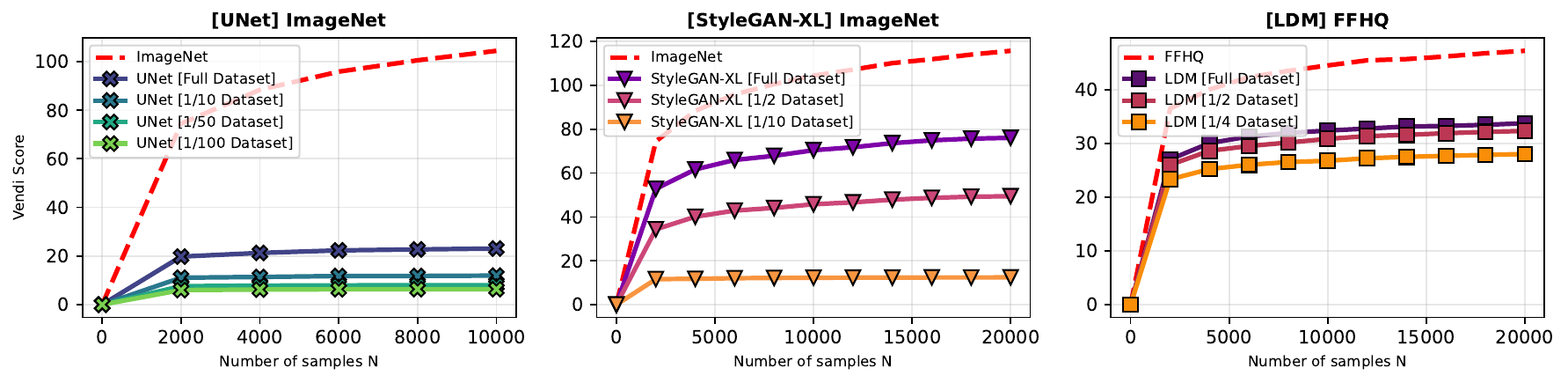}
    \caption{Vendi Score comparison for UNet, StyleGAN-XL, LDM. The models were trained on varying subsets of ImageNet, ImageNet, FFHQ respectively (indicated by fraction of the original size). Dashed red line represents score of the corresponding training dataset. The evaluation results are done using CLIP with $\sigma=8.0$.}
    \label{fig:clip-all-trained-models}
\end{figure}

\subsubsection{Extending results to reference-based metrics (Recall and Coverage)}

We further evaluated the LDM and StyleGAN-XL models using Recall, Coverage, RKE, and Vendi scores. The results align with our reference-free analysis, as both Recall and Coverage confirm that diversity degrades with smaller training set sizes.

\begin{table}[ht]
\centering
\caption{Comparison of RKE, Coverage, and Recall across StyleGAN-XL models trained on varying sizes of ImageNet.}
\label{tab:train_stylegan_rke_vs_recall_coverage}
\begin{tabular}{lcccc}
\toprule
Method & Recall $\uparrow$ & Coverage $\uparrow$ & RKE $\uparrow$ & Vendi $\uparrow$\\
\midrule
Dataset & \textbf{-} & \textbf{-} & \textbf{36.82} & \textbf{2533.93}\\
\midrule
StyleGAN-XL: full dataset & 0.74 & 0.96 & 31.27 & 2203.39\\
StyleGAN-XL: 1/2 dataset & 0.59 & 0.94 & 20.80 & 1432.20\\
StyleGAN-XL: 1/10 dataset & 0.42 & 0.93 & 14.95 & 255.44\\
\bottomrule
\end{tabular}
\end{table}

\begin{table}[ht]
\centering
\caption{Comparison of RKE, Coverage, and Recall across LDM models trained on varying sizes of FFHQ.}
\label{tab:train_ldm_rke_vs_recall_coverage}
\begin{tabular}{lcccc}
\toprule
Method & Recall $\uparrow$ & Coverage $\uparrow$ & RKE $\uparrow$ & Vendi $\uparrow$\\
\midrule
Dataset & \textbf{-} & \textbf{-} & \textbf{14.46} & \textbf{514.02}\\
\midrule
LDM: full dataset & 0.44 & 0.74 & 10.21 & 294.60\\
LDM: 1/2 dataset & 0.22 & 0.74 & 9.17 & 260.64\\
LDM: 1/4 dataset & 0.03 & 0.52 & 7.06 & 170.44\\
\bottomrule
\end{tabular}
\end{table}

\goodbreak\newpage

\subsection{Additional guidance results}
\subsubsection{Hyperparameters and Experimental settings}
In the SPARKE and Vendi Guidance experiment, we considered a Gaussian kernel, which consistently led to higher output scores in comparison to the other standard cosine similarity kernel. We used the same Gaussian kernel bandwidth $\sigma$ in the RKE and Vendi experiments, and the bandwidth parameter choice matches the selected value in \cite{jalali2024conditional, friedman2023vendi}. The numerical experiments were conducted on 4$\times$NVIDIA GeForce RTX 4090 GPUs, each of which has 22.5 GB of memory.

\subsubsection{Experimental Configuration for Table~\ref{tab:ditxl_vendi_rke_kd_fd}}

\textbf{Vendi Guidance.} We used a Gaussian kernel with bandwidth \(\sigma_{img} = 0.8\) and used $\eta=0.03$ as the weight of Vendi guidance. To balance the effects of the diversity guidance in sample generation, the Vendi guidance update was applied every 10 reverse-diffusion steps in the diffusion process, which is similar to the implementation of Vendi score guidance in \cite{hemmat2024cvsg}.

\textbf{SPARKE (RKE Guidance).} We considered the same Gaussian kernel for the image generation with bandwidth \(\sigma_{img} = 0.8\). The guidance hyperparameter was set to $\eta=0.03$, as in Vendi guidance. The SPARKE diversity guidance was applied every 10 reverse-diffusion steps.

\subsubsection{SPARKE and Vendi Guidance}
As discussed in the main text, one effective strategy to mitigate downward diversity bias is to apply regularization during training. We extend the analysis to a circular and grid pattern configurations. Our observations remain consistent: regularization significantly improves the mode coverage of the JS-GAN model. Furthermore, Table~\ref{tab:gaussian-mixture-vendi_rke_kd_fd} shows that this improvement extends to reference-based metrics, enhancing both quality and diversity scores.

\begin{figure}
    \centering
    \includegraphics[width=0.8\textwidth]{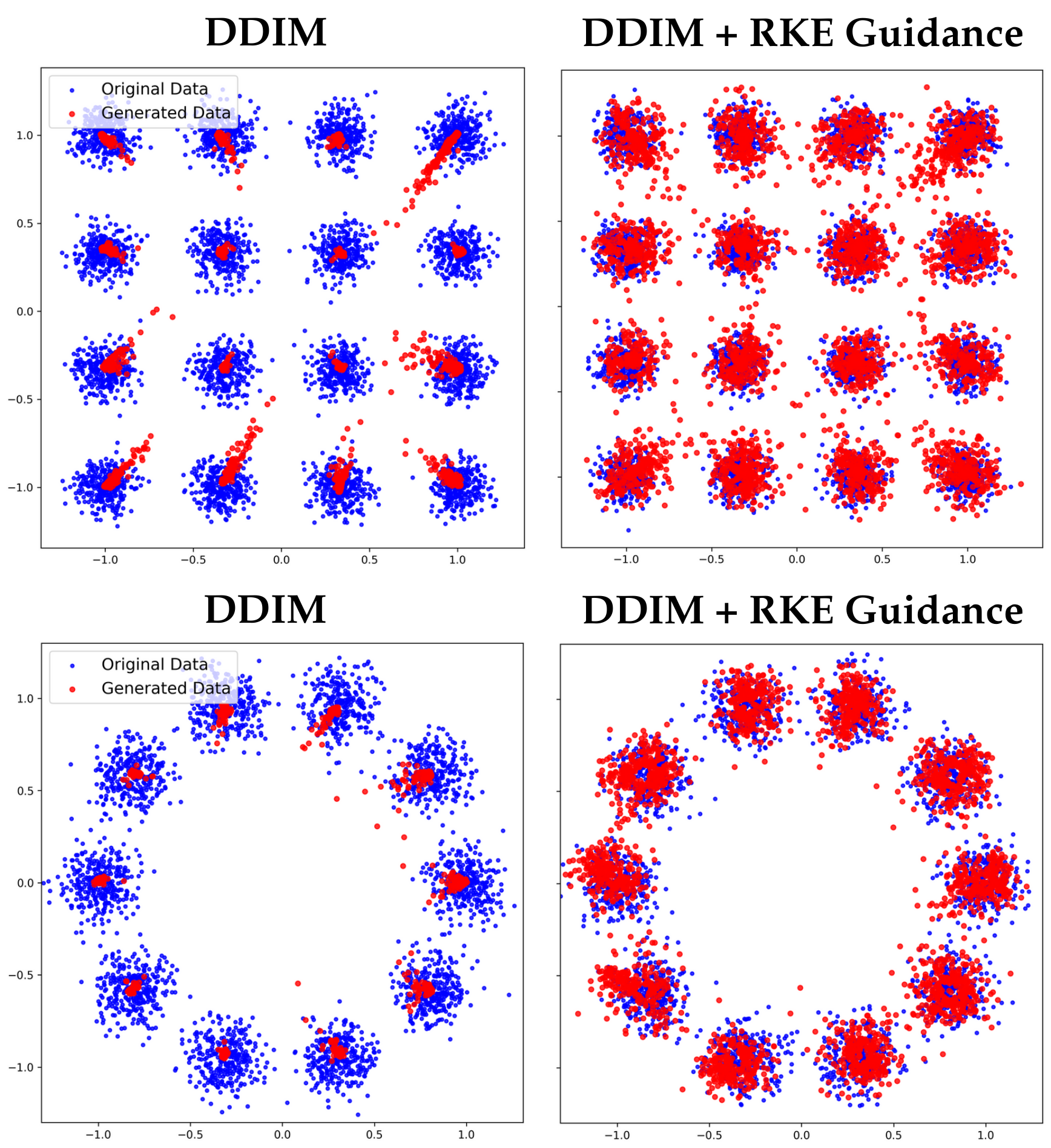}
    \caption{Comparison of DDIM with and without RKE regularization in two 2D GMM distributions: (1) modes are placed in a circular pattern, (2) modes are placed in a grid.}
    \label{fig:ddim-guidance-full}
\end{figure}

\begin{table}[ht]
\centering
\caption{Comparison of Vendi scores, RKE, KD, and FD across different sampling settings. on 2D Gaussian Mixtures with 10 components.}
\label{tab:gaussian-mixture-vendi_rke_kd_fd}
\begin{tabular}{lccccccccc}
\toprule
Method & Vendi & RKE  & KD $\times$ 10$^{2}$ & FD & Precision & Recall & Density & Coverage \\
\midrule
Dataset        & 16.49 & 12.86 & - & - & - & - & - & -\\
JS-GAN+RKE Reg.   & 15.20 & 12.18 &  0.098  & 0.081 & 0.98 & 0.89 & 1.02 & 0.78 \\
JS-GAN                       &  12.61 & 10.92 &  0.767 & 0.654  & 0.98  & 0.83  & 0.99 & 0.49 \\
\bottomrule
\end{tabular}
\end{table}

\begin{figure}
    \centering
    \includegraphics[width=0.8\textwidth]{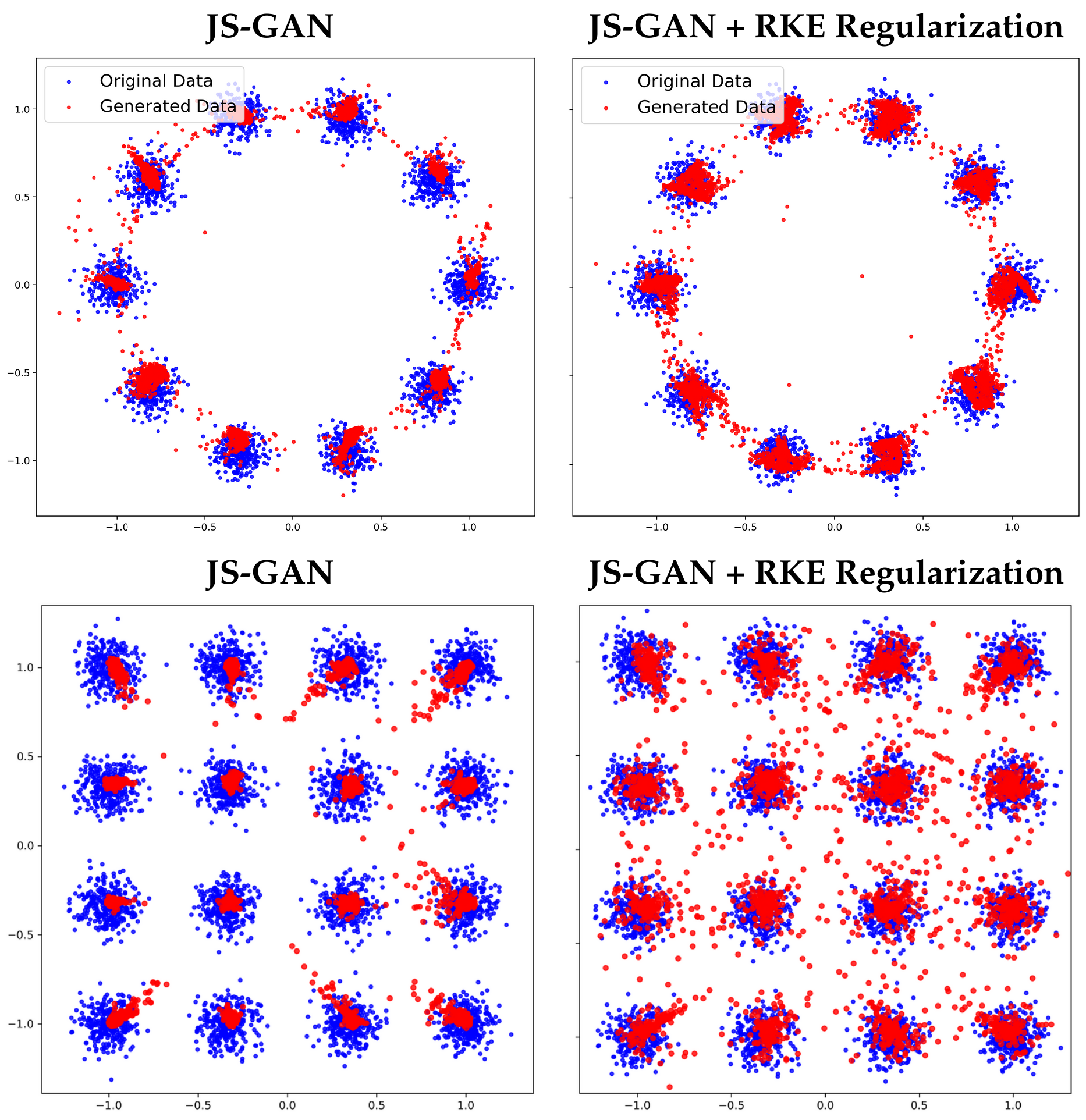}
    \caption{Comparison of JS-GAN generated GMM modes with and without diversity regularization during training. We show two mode placement examples: (1) modes are in a circular pattern, (2) modes are in the $4\times 4$ grid pattern.}
    \label{fig:js-gan-regularization}
\end{figure}

\begin{figure}
    \centering
    \includegraphics[width=0.8\textwidth]{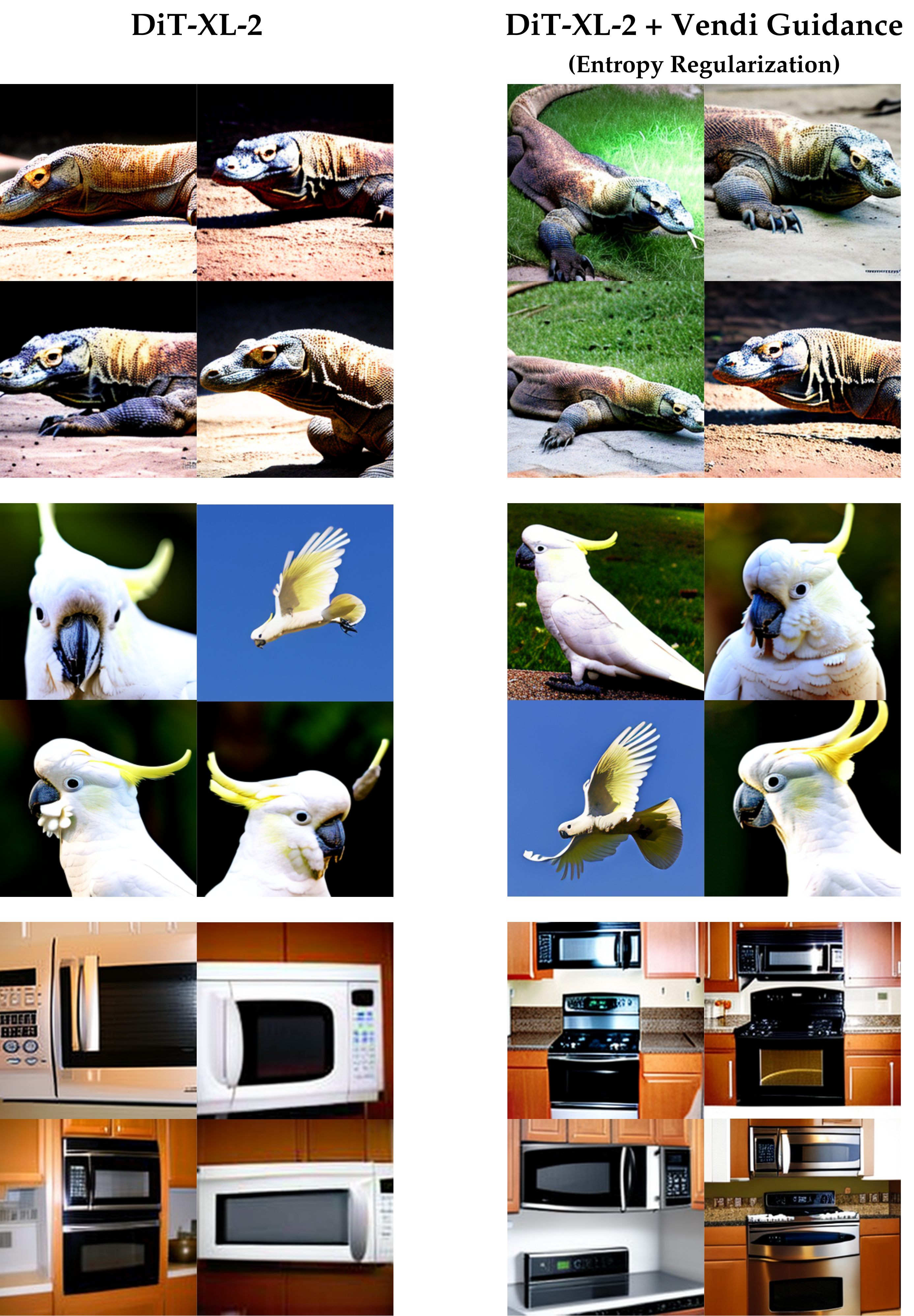}
    \caption{Qualitative Comparison of samples generated by DiT-XL-2 and DiT-XL-2 with Vendi guidance (entropy regularization).}
    \label{fig:ditxl-examples-guidance}
\end{figure}

Moreover, to mitigate the downward diversity bias, we demonstrate the efficacy of inference-time guidance on Diffusion models. By leveraging Vendi and RKE-based guidance, we actively encourage the model to explore underrepresented regions of the learned manifold. Table~\ref{tab:ditxl_vendi_rke_kd_fd} reports that applying Vendi or RKE guidance yields consistent improvements across Vendi, RKE, Recall, and Coverage scores. Crucially, this gain in diversity does not come at the cost of quality. Instead, we observe a concurrent improvement in sample quality—indicated by lower FD and KD, and higher Precision and Density—suggesting that the guidance effectively corrects.

\begin{table}[ht]
\centering
\caption{Comparison of Vendi scores, RKE, KD, and FD across different sampling settings. on DiT (trained on ImageNet) using DINOv2 feature embeddings. ($\sigma=35$)}
\label{tab:ditxl_vendi_rke_kd_fd}
\begin{tabular}{lccccccccc}
\toprule
Method & Vendi $\uparrow$ & RKE $\uparrow$ & KD $\downarrow$ & FD $\downarrow$ & Precision $\uparrow$ & Recall $\uparrow$ & Density$\uparrow$ & Coverage$\uparrow$ \\
\midrule
Dataset        & 1747.60 & 36.37 & - & - & - & - & - & -\\
Vendi guidance &  672.24 & 33.94 &  0.1628 & 328.16 & 0.95 & 0.15 & 1.63 & 0.73\\
RKE guidance   & 751.09  & 34.50  &  0.1629  & 319.56 & 0.96 & 0.16 & 1.54 & 0.74\\
No guidance    &  603.68 & 33.71 &  0.2174 & 422.15  & 0.94 & 0.11  & 1.50 & 0.64 \\
\bottomrule
\end{tabular}
\end{table}


 \end{appendices}

\end{document}